\documentclass[10pt,twoside]{IEEEtran}
\usepackage{amsopn}
\usepackage{amsthm}
\usepackage{etex}
\usepackage{spconf,graphicx}
\usepackage{endnotes}
\usepackage{hyperref}
\usepackage{epsfig,psfrag}
\usepackage{pst-all}
\usepackage{amssymb,amsfonts,upref,cite,epsf,color,bm}
\usepackage{graphicx}
\usepackage{color}
\usepackage{csvsimple}
\usepackage{url}
\usepackage{amsmath}
\usepackage{graphicx}
\usepackage{calc}
\usepackage{booktabs}
\usepackage{tikz,stackengine}
\usepackage{pgfplots}

\usepackage{etex}
\usepackage{algorithm, algpseudocode}
\usepackage{endnotes}

\usepackage{epsfig,psfrag}

\usepackage{color}
\usepackage{calc}

\usepackage{csvsimple}
\usepackage{xinttools} 
\usepackage{tikz,pgfplots}
\usetikzlibrary{calc}
\usetikzlibrary{shapes.misc}

\newtheorem{assumption}{Assumption}
\usepackage{framed}

\newcommand\norm[2][\Tnorm]{\ensuremath{{\|#2\|}_{#1}}}

\newcommand{\set}[1]{\mathcal{#1}} 
\newcommand\defeq{:=}
\newcommand{\featurelen}{p}

\newcommand{\setT}{\mathcal{T}}

\newcommand\vect[1]{\mathbf #1}

\newcommand{\setofT}{\Sigma^{p}_{s}}

\newcommand{\va}{\vect{a}}

\newcommand{\vu}{\vect{u}}  
\newcommand{\vv}{\vect{v}}  
\newcommand{\vw}{\vect{w}}
\newcommand{\vx}{\vect{x}}

\newcommand{\procdim}{p}

\newcommand{\mC}{\vect{C}}

\newcommand{\mK}{\vect{K}}

\newcommand{\mP}{\vect{P}}

\def\baselinestretch{1}
\renewcommand{\baselinestretch}{1.6}\small\normalsize
\def \expect {{\rm E} }
\def \prob {{\rm P} }
\def \twiddle[#1] {e^{-j \frac{2 \pi}{N}  #1 }}
\def \twiddleneg[#1] {e^{j \frac{2 \pi}{N}  #1 }}

\def\x{{\mathbf x}}

\DeclareMathOperator*{\rank}{rank}

\DeclareMathOperator*{\argmin}{arg\;min}

\DeclareMathOperator*{\trace}{tr}

\def\ML_est{\hat{\mathbf{x}}_{\text{ML}}}

\newcommand{\cig}{\mathcal{G}}
\newcommand{\timevar}{n}
\newcommand{\timeidx}{\timevar}
\newcommand{\freqidx}{k}

\newcommand{\lagvar}{m}

\newcommand{\samplesize}{N}
\newcommand\coefflen{p}
\newcommand\coeffdim{p}
\newcommand\sparsity{s}

\newcommand{\nrblocks}{B}
\newcommand{\blocklen}{L}
\newcommand{\errevent}{\mathcal{E}_{\mathcal{T}}}
\newcommand{\error}{\mathcal{E}_{i}}
\newcommand{\blkidx}{b}

\newcommand{\blockLength}{L}

\newcommand{\edges}{\set{E}}
\newcommand{\nodes}{\set{V}}

\newcommand{\blknot}{{(\blkidx)}}

\newcommand{\nat}[1]{[#1]}

\newcommand{\component}[1]{\vx_{#1}}

\newcommand{\CMX}[1]{\mC[#1]}

\newcommand{\bfep}{\bm \varepsilon}

\newcommand{\be}{\begin{equation}}
\newcommand{\ee}{\end{equation}}

\newcommand{\equref}[1]{(\ref{#1})}

\usepackage{algorithm}
\usepackage{algpseudocode}
\floatname{algorithm}{Algorithm}
\algnewcommand\algorithmicinput{\textbf{Input:}}
\algnewcommand\INPUT{\item[\algorithmicinput]}
\algnewcommand\algorithmicoutput{\textbf{Output:}}
\algnewcommand\OUTPUT{\item[\algorithmicoutput]}
\newtheorem{theorem}{Theorem}

\newtheorem{lemma}[theorem]{Lemma}

\newtheorem{corollary}[theorem]{Corollary}

\allowdisplaybreaks

\makeatother

\usepackage{setspace}

\newlength{\depthofsumsign}
\algnewcommand\algorithmicforeach{\textbf{for each}}
\algdef{S}[FOR]{ForEach}[1]{\algorithmicforeach\ #1\ \algorithmicdo}

\title{On the Sample Complexity of Graphical Model Selection from Non-Stationary Samples}
\name{\hspace*{-10mm} Nguyen Tran$^{1}$,  Oleksii Abramenko$^{1}$, Alexander Jung$^{1}$, 
}

\address{\normalsize $^1$Department of Computer Science, Aalto University, Espoo, Finland; firstname.lastname(at)aalto.fi \\
}

\begin{document}

\maketitle


\begin{abstract}
We characterize the sample size required for accurate graphical model selection for a system which is 
observed via samples (measurements) forming a non-stationary vector-valued time series. In particular, the observed 
data is modelled as a vector-valued zero-mean Gaussian random process whose samples are uncorrelated 
but have different covariance matrices. This model contains as special cases the standard setting of i.i.d.\ samples 
as well as the case of samples forming a stationary time series. More generally, our approach applies to time 
series data for which efficient decorrelation transforms, such as the Fourier transform for stationary time series, 
are available. By analyzing a particular model selection method, we derive a sufficient condition 
on the required sample size for accurate graphical model selection based on non-stationary data. 
\end{abstract}

\vspace*{-2mm}
\section{Introduction}
\label{sec_intro} 

\vspace{-1mm}
A powerful approach to managing massive datasets (big data) is based on network or graph representations of the datasets 
\cite{SandrMoura2014b,NetworkLasso,koller2009probabilistic,BigDataNetworksBook}. Examples of networked data are found 
in signal processing where signal samples can be arranged as a chain, in image processing with pixels arranged on a grid, in 
wireless sensor networks where measurements conform to sensor proximity \cite{SandrMoura2014b}. Organising data using 
networks is also used in knowledge bases (graphs) whose items are linked by relations \cite{WikiData2014,Sadeghi2017}.

Using network models is beneficial from a computational and statistical perspective. Indeed, network models for data lend naturally 
to highly scalable learning algorithms in the form of message passing on the data network \cite{DistrOptStatistLearningADMM}. 
Moreover, the network structure allows to borrow statistical strength across different localized high-dimensional statistical models 
which are associated with individual data points (nodes) \cite{NetworkLasso,Ambos2018}.  Finally, network models provide a high level of 
flexibility in order to cope with heterogeneous datasets composed of different data types (e.g., mixtures of audio, video and text data).

In some applications, the network structure underlying the data is not known explicitly but has to be learned in a data-driven fashion. 
This task can be accomplished in a principled way by using probabilistic graphical models (PGM) \cite{koller2009probabilistic,GraphModExpFamVarInfWainJor}. 
Within a PGM, we interpret data points as realizations of random variables. A particular type of PGM is based on representing the conditional 
independence relations between individual data points using a network structure (graph) \cite{LauritzenGM,GraphModExpFamVarInfWainJor}. 
The problem of estimating the network structure of a PGM from observed data is known as graphical model selection (GMS). 

Many efficient methods have been proposed for GMS for data which is modelled as sequences of i.i.d.\ realizations of some 
underlying random vector \cite{RavWainLaff2010,FriedHastieTibsh2008,JMLRHub}. The extension of GMS from the i.i.d.\ 
setting to cope with correlations between vector samples using stationary process models has been studied in 
\cite{BachJordan04,CSGraphSelJournal,HannakJung2014conf,JuHeck2014,JungGaphLassoSPL}. A robust GMS method 
which is able to cope with outliers is proposed in \cite{Yang2015Nips}. In this paper, we consider the extension of 
GMS to non-stationary time series data. As we will detail below, our approach includes GMS for stationary time series 
as a special case. 

It is of practical relevance for the usage of GMS methods to understand the fundamental requirements on the available data 
such that accurate GMS is possible. For data which can be modelled as i.i.d.\ realizations of a Gaussian random vector (Gaussian Markov random field), 
the required sample size is well understood. A lower bound on the sample size has been obtained by \cite{WangWain2010}, 
which does not place any computational constraints on the GMS method. Remarkably, this lower bound can be achieved by 
computationally tractable convex optimizaton methods \cite{RavWainRaskYu2011} proving them as optimal in terms of sample 
size requirement. By adapting the information-theoretic approach of \cite{WangWain2010}, a lower bound on the sample size 
required for accurate GMS from data conforming to a stationary random process model is presented in \cite{HannakJung2014conf}. 

{\bf Contribution.} Our focus is on the required sample (data) size which allows for accurate GMS. In contrast to most existing work, 
we study GMS for data which cannot be well modelled as a stationary random process. To this end, we propose a simple but useful 
probabilistic model for non-stationary data whose statistical properties vary over time or space (see Section \ref{SecProblemFormulation}). 
This model requires that samples can be grouped into blocks (of known size) within which the samples can be considered as i.i.d. Our model includes, 
as important special cases, the case of i.i.d.\ data as well as data forming a stationary time series. Moreover, the model also applies to 
data which can be represented as either cyclostationary \cite{Kipnis2018}, locally stationary \cite{Mallat98} or underspread random 
processes \cite{TimeFrequencyAnalysisBoashash}. Thus, in contrast to existing GMS methods \cite{BachJordan04,CSGraphSelJournal,JungGaphLassoSPL,Dahlhaus2000,Eichler03} 
which require stationary time series, we consider GMS from non-stationary time series data. 


In general, the process model used in this paper is applicable whenever an efficient decorrelation 
transformation, which allows one to turn the raw data into blocks of i.i.d. random vectors, is available. 
An important example is the Fourier transform of a stationary time series which can be well 
approximated as block-wise i.i.d.\ samples (see Section \ref{SecProblemFormulation}). 
Our process model has also been used in \cite{DanaherGroupGLASSO2014,LeeLiu2015} in the context of a bioinformatics 
application. However, while \cite{DanaherGroupGLASSO2014,LeeLiu2015} aims at learning different graphical models 
for each block, we are interested in learning a single global graphical model for all blocks. 

The main focus of \cite{DanaherGroupGLASSO2014,Peterson2013} is the design of computationally 
feasible GMS methods (e.g., based on convex optimization). Instead, our aim is not the design of a 
computationally tractable (``polynomial time'') GMS method but rather a characterization of the 
required amount of data (sample size) for reliable GMS. To this end we provide a careful analysis 
of a computationally intractable neighborhood regression method which amounts to an exhaustive 
search for conditional dependencies between two particular data points (represented by two nodes 
in the PGM), when conditioning on all remaining data points. 

Our conceptual approach to GMS extends the sparse neighbourhood regression approach 
put forward in \cite{MeinBuhl2006} for GMS from i.i.d.\ samples to the non-stationary setting. 
However, while \cite{MeinBuhl2006} proposes a computationally attractive convex relaxation of 
sparse neighbourhood regression using a Lasso-based estimator, we are mainly interested in the 
fundamental limits on the required sample size without constraining the computational complexity 
of the GMS method.

The main contribution of this work is a precise characterization of the sample size required for 
accurate GMS from non-stationary data. In particular, we show that the required sample size 
depends crucially on the minimum average connection strength between the individual  process 
components. If this quantity is sufficiently large, accurate GMS is possible even in the high-dimensional 
regime, where the length of the vector samples might (drastically) exceed the number of available 
training samples (data points).  

{\bf Outline.} After formalizing the problem setup in Section~\ref{SecProblemFormulation}, we analyze 
a simple GMS method, which we term sparse neighbourhood regression, in  Section~\ref{sec_GMS_neighbrohoud_regr}. 
In particular, for a given sample size and sparsity level of the network structure, we derive an upper bound on the 
probability that sparse neighbourhood regression fails in recovering the correct network structure of the PGM. 
This upper bound on the error probability implies an upper bound on the required sample size such that GMS is 
feasible. We verify our theoretical findings by means of numerical experiments in Section \ref{sec_numerical_results}. 

\paragraph*{Notation} 
For a vector $\vx=(x_{1},\ldots,x_{d})^{T} \in \mathbb{R}^{d}$, the Euclidean and $\infty$-norm are $\|\vx\|_{2} \!\defeq\! \sqrt{\vx^T\vx}$ and 
$\| \vx \|_{\infty} \!\defeq\! \max_{i} |x_{i}|$, respectively. The $m$-th largest eigenvalue of a positive semidefinite (psd) 
matrix $\mathbf{C}$ is $\lambda_{m}(\mC)$. Given a matrix $\mathbf{Q}$, we denote its transpose, trace, rank, 
spectral norm and Frobenius norm by $\mathbf{Q}^{T}$, $\trace\{\mathbf{Q}\}$, $\rank\{\mathbf{Q}\}$, $\| \mathbf{Q} \|_{2}$ 
and $\| \mathbf{Q} \|_{\rm F}$, respectively. For a finite sequence of matrices $\mathbf{Q}_{l} \in \mathbb{R}^{d \times d}$, 
with $l=1,\ldots \nrblocks$, we denote by ${\rm blkdiag} \{ \mathbf{Q}_{l} \}$ the block diagonal matrix of size $\nrblocks d \times \nrblocks d$ 
with the $l$th diagonal block given by $\mathbf{Q}_{l}$. The identity matrix of size $d \times d$ is $\mathbf{I}_{d}$. The minimum (maximum) of two numbers $a$ 
and $b$ is denoted $a\!\wedge\!b$ ($a\!\vee\!b$). The set of non-negative real (integer) numbers is 
denoted $\mathbb{R}_{+}$ ($\mathbb{Z}_{+}$). The probability of an event $\mathcal{E}$ is $\prob\{\mathcal{E}\}$. 
The complement of an event $\mathcal{A}$ is denoted $\mathcal{A}^{c}$. The expectation of a random variable $y$ is $\expect\{y\}$. 

\vspace{-2mm}
\section{Problem Formulation}
\label{SecProblemFormulation}
\vspace{-1mm}

We consider a system which is constituted by $\procdim$ components $\vx_{i}$, for $i=1,\ldots,\procdim$. In a 
bioinformatics application, such a system might be a gene regulatory network with the components $\vx_{i}$ 
representing concentrations of particular genes \cite{DavidsonLevin2005}. The system is observed by acquiring 
$\samplesize$ vector-valued samples $\{\vx[\timevar]\}_{\timevar=1}^{\samplesize}$, each sample 
\begin{equation}
\nonumber
\vx[\timevar] = \big(x_{1}[\timevar],\ldots,x_{\featurelen}[\timevar]\big)^{T} \!\in\! \mathbb{R}^{\procdim}
\end{equation} 
constituted by $\procdim$ scalar ``measurements'' $x_i[\timevar]$ for $i =1,\ldots,\featurelen$. 

The samples $\vx[\timeidx]$ are modelled as realizations of zero-mean Gaussian random vectors, which 
are uncorrelated such that 
\begin{equation}
\nonumber 
\expect \big\{ \vx[\timevar] \big(\vx[\timevar']\big)^{T} \big\} \!=\! \mathbf{0} \mbox{ for }\timevar \!\neq\! \timevar'.
\end{equation}  
The probability distribution of the samples $\vx[\timeidx]$ is fully specified by 
the covariance matrices 
\begin{align}
\label{equ_def_cov_matrix}
\mC[\timevar] \defeq \expect\big\{ \vx[\timevar] \big( \vx[\timevar] \big)^{T} \big\}. 
\end{align}
In general, the covariance matrix $\mathbf{C}[\timevar]$ varies with sample index $\timevar$, i.e., 
$\mathbf{C}[\timevar] \neq \mathbf{C}[\timevar']$ for $\timevar \!\neq\! \timevar'$ in general. However, 
we do not allow for arbitrary variation of the covariance matrix but require it to be constant over 
blocks of $\blocklen$ consecutive samples $\vx[\timevar],\ldots,\vx[\timevar\!+\!\blocklen\!-\!1]$. 
We model the observed samples as blocks of i.i.d.\ Gaussian random vectors, 
\vspace{-1mm}
\begin{align} 
\label{equ_proc_model} 
\underbrace{\vx[1], \ldots,\vx[\blocklen]}_{{\rm i.i.d.} \sim \mathcal{N}(\mathbf{0},\mathbf{C}^{(\blkidx=1)})},\ldots,\ldots,\ldots,
\underbrace{\vx[\samplesize\!-\!\blocklen\!+\!1],\ldots,\vx[\samplesize]}_{{\rm i.i.d.} \sim \mathcal{N}(\mathbf{0},\mathbf{C}^{(\nrblocks)})}.
\end{align}

Our goal is to estimate the conditional dependencies between the components $\vx_{i}$ which are represented by 
the sequences $x_{i}[1],\ldots,x_{i}[\samplesize]$ in \eqref{equ_proc_model}. Such a global dependence structure 
between entire sequences has also been considered in \cite{BachJordan04}. However, \cite{BachJordan04} 
considered stationary time series data, we consider global dependence structure between quantities that are 
modelled as a non-stationary process of the form \eqref{equ_proc_model} (which contains the Fourier transform 
of stationary time series as a special case). 

The vector samples $\vx[\timevar]$ in \eqref{equ_proc_model} are uncorrelated (independent) zero-mean Gaussian 
vectors with covariance matrix
\begin{equation}
\label{equ_cov_matrix_blocks}
\mathbf{C}[\timevar] \!=\! \mathbf{C}^{(\blkidx)} \mbox{ for } \timevar \in \{(\blkidx-1)\blockLength+1, \ldots, \blkidx \blockLength\}.
\end{equation}
For ease of exposition and without essential loss of generality, we henceforth assume the sample size $\samplesize$ 
to be a integer multiple of the block length $\blocklen$ (which is assumed fixed and known), such that $\samplesize\!=\!\nrblocks \blocklen$, 
with the number $\nrblocks$ of data blocks. Moreover, we tacitly assume the covariance matrices $\mC[\timevar]$ 
to be non-singular (invertible) with inverse $\big( \mC[\timevar] \big)^{-1}$ (see Assumption \ref{aspt_eig_val} below).

The model \eqref{equ_proc_model} reduces to the i.i.d.\ setting for $\nrblocks = 1$ and block length $ \blockLength = \samplesize$. 
In this paper, we study the fundamental limits of accurate GMS based on non-stationary data which conforms to 
the model \eqref{equ_proc_model} with $\nrblocks>1$ (the non-stationary setting).

At first glance, the process model \eqref{equ_proc_model} might seem overly restrictive as it still requires blocks of 
consecutive samples to be i.i.d. However, as we will now discuss, the model \eqref{equ_proc_model} can be used 
as an approximation at least for some important classes of random processes. For each 
of these process classes we are able to identify useful choices for the block length $L$ in \eqref{equ_proc_model}.  

{\bf Stationary Processes.} The model \eqref{equ_proc_model} covers the case where the observed 
samples form a stationary process \cite{CSGraphSelJournal,Dahlhaus2000,BachJordan04,JungGaphLassoSPL}. 
Indeed, consider a zero-mean Gaussian stationary process $\vx[\timevar]$ with auto-covariance function 
\begin{equation} 
\mathbf{R}_{x}[\lagvar] \defeq \expect \big\{ \vx[\timevar] \big(\vx[\timevar\!-\!\lagvar]\big)^{T} \big\}
\end{equation}
and spectral density matrix (SDM)  \cite{Dahlhaus2000}
\begin{equation} 
\mathbf{S}_{x}(\theta) \defeq \sum_{\lagvar=-\infty}^{\infty} \mathbf{R}_{x}[\lagvar] \exp(- j 2 \pi \theta \lagvar).
\end{equation} 
Let 
\begin{equation} 
\nonumber
\hat{\vx}[\freqidx]\defeq (1/\sqrt{\samplesize}) \sum_{\timevar=1}^{\samplesize} \vx[\timevar] \exp( - j 2 \pi (\timevar\!-\!1)(\freqidx\!-\!1)/\samplesize)
\end{equation} 
denote the discrete Fourier transform (DFT) of the stationary process $\vx[\timeidx]$. Then, by well-known properties 
of the DFT (see, e.g., \cite{Brockwell91}), the vectors $\hat{\vx}[\freqidx]$, for $\freqidx=1,\ldots,\samplesize$, are 
approximately uncorrelated Gaussian random vectors with zero mean and covariance matrix $\mathbf{C}[\freqidx] \approx \mathbf{S}_{z}(\freqidx/\samplesize)$. 
For a stationary process $\vx[\timeidx]$ with (effective) correlation width $W$, the SDM is approximately constant (flat) 
over a frequency interval of length $1/W$. Thus, the DFT vectors $\hat{\vx}[\freqidx]$ approximately conform to the process 
model \eqref{equ_proc_model} with block length $\blocklen=\samplesize/W$ (since the DFT vectors correspond to SDM 
samples at evenly spaced frequency points separated by $1/\samplesize$).

{\bf Cyclostationary Processes.} As detailed in \cite{Kipnis2018}, (discrete-time) cyclostationary processes can be transformed 
to vector-valued (or multivariate) stationary processes which can then, in turn, be transformed to a process of the form 
\eqref{equ_proc_model} via a DFT. 

{\bf Locally Stationary Processes.} The process model \eqref{equ_proc_model} applies to locally stationary 
processes \cite{Starica2005,Wahlberg2007,Mallat98}. The i.i.d.\ blocks of $\blocklen$ consecutive vector 
samples in \eqref{equ_proc_model} correspond to the homogeneity intervals defined in \cite{Starica2005}. 
Particular approaches for optimally chosing the block length $L$ for the model \eqref{equ_proc_model} are 
studied in \cite{Starica2005,Dahlhaus98}. One important example of locally stationary processes 
are time-varying autoregressive processes which extend traditional autoregressive process models by 
allowing time-varying regression coefficients \cite{Dahlhaus2009,Brockwell91}. 

{\bf Underspread Processes.} The process model \eqref{equ_proc_model} can be used for  
underspread non-stationary processes \cite{jung-specesttit,TimeFrequencyAnalysisBoashash}. 
A continuous-time random process $\vx(t)$ is underspread if its expected ambiguity function (EAF) 
\begin{align*}
\bar{\mathbf{A}}_{x}(\tau,\nu) \!\defeq\! \hspace{-1mm} \int\limits_{t\!=\!-\infty}^{\infty} \hspace{-3mm} \expect \big\{ \vx(t\!+\!\tau/2) \vx^{T}(t\!-\!\tau/2)  \big\} \exp(- j 2 \pi t \nu) dt
\end{align*} 
is well-concentrated around the origin in the $(\tau,\nu)$ plane. In particular, if the EAF $\bar{\mathbf{A}}_{x}(\tau,\nu)$ of $\vx(t)$ is (effectively)
supported on the rectangle $[-\tau_{0}/2,\tau_{0}/2] \times [-\nu_{0}/2,\nu_{0}/2]$, then the process $\vx(t)$ is underspread if $\tau_{0} \nu_{0} \ll 1$. 

It can be shown that for a suitably chosen prototype function $g(t)$ (e.g., a Gaussian pulse) and grid 
constants $T$ and $F$, the Weyl-Heisenberg set $\big\{ g^{(n,k)}(t) \defeq g(t\!-\!nT)e^{-2 \pi k F t} \big\}_{n,k\in \mathbb{Z}}$, yields zero-mean analysis 
coefficients $\hat{\vx}[n,k] = \int_{t=-\infty}^{\infty} \vx(t) g^{(n,k)}(t) dt$ which are approximately uncorrelated and provide a complete representation 
of the process $\vx(t)$. Moreover, the covariance matrix of $\hat{\vx}[(n,k)]$ is approximately equal to the value $\overline{\mathbf{W}}_{x}(nT,kF)$ 
of the Wigner-Ville spectrum (WVS) \cite{Velez90}
\begin{align*}
\overline{\mathbf{W}}_{x}(t,f) \defeq \hspace*{-2mm} \int\limits_{\tau,\nu=-\infty}^{\infty} \bar{\mathbf{A}}_{x}(\tau,\nu)  \exp(- 2 \pi (f \tau - \nu t)) d \tau d \nu
\end{align*}
which can be loosely interpreted as a time-varying power spectral density.
For an underspread process whose EAF is effectively supported on $[-\tau_{0}/2,\tau_{0}/2] \times [-\nu_{0}/2,\nu_{0}/2]$, 
the WVS $\overline{\mathbf{W}}_{x}(nT,kF)$ is approximately constant over a rectangle of area $\approx 1/(\tau_{0} \nu_{0})$. 
Thus, the vectors $\hat{\vx}[(n,k)]$ approximately conform to the process model \eqref{equ_proc_model} with block length 
$\blocklen \!\approx\! \frac{1}{TF\tau_{0} \nu_{0}}$. 

{\bf Conditional Independence Graph.} 
We now define a PGM for the observed samples $\{ \vx[\timeidx]\}_{\timeidx=1}^{\samplesize}$ (cf. \eqref{equ_proc_model}) by identifying the 
individual components 
\begin{equation}
\label{equ_component_i}
\component{i}=(x_{i}[1],\ldots,x_{i}[\samplesize])^{T}
\end{equation} 
with the nodes $\nodes\!=\!\{1,\ldots,\procdim\}$ of an undirected simple graph $\cig=(\nodes,\edges)$ (see Figure \ref{fig_process_model}). 
This graph encodes conditional independence relations between the components $\component{i}$ and is hence called the 
conditional independence graph (CIG) of the process $\vx[\timevar]$. In particular, an edge is absent between nodes $i, j \in \nodes$, i.e., $\{i,j\} \notin \edges$,  
if the corresponding process components $\component{i}$ and $\component{j}$ 
are conditionally independent, given the remaining components $\{ \component{r} \}_{r \in \nodes \setminus \{i,j\}}$. 

We highlight that the CIG $\cig$ represents stochastic dependencies between the components $ \big\{ \component{i} \big\}_{i=1}^{\featurelen}$ 
(see \eqref{equ_component_i}) of the vector samples $\vx[1],\ldots,\vx[\samplesize]$ in a global fashion, i.e., jointly for all 
$\timevar\!=\!1,\ldots,\samplesize$. In particular, the edge set $\edges$ does not depend on the sample index $\timevar$ since 
we define the CIG for the entire sample process for $\timevar=1,\ldots,\samplesize$.\footnote{In principle, it is also possible to define 
a CIG $\cig^{(\timevar)}$ separately for each individual sample $\vx[\timevar] = \big(x_{1},\ldots,x_{\featurelen}[\timevar]\big)^{T}$,
which can be interpreted as a single realization of a Gaussian random Markov field. The edge set of the global CIG we are considering 
in this paper is the union of the edge sets in the sample-wise CIGs $\cig^{(\timevar)}$.}
Using a global CIG between data points which are modelled as non-stationary vector samples is useful for many applications 
(see \cite{BachJordan04,DanaherGroupGLASSO2014,Dahlhaus2000,Eichler03} and references therein). 

\begin{figure}
\centering
\begin{tikzpicture}[x=0.5cm,y=1cm]
 \node[draw, rounded corners,thick] (x1) at (3,4) {$\vx_1\!=\!\big(x_{1}[1],\ldots,x_{1}[\samplesize]\big)^{T}$};
  \node[draw, rounded corners,thick] (x2) at (-1.5,1) {$\vx_2\!=\!\big(x_{2}[1],\ldots,x_{2}[\samplesize]\big)^{T}$};
   \node[draw, rounded corners,thick](x3) at (7,1) {$\vx_3\!=\!\big(x_{3}[1],\ldots,x_{3}[\samplesize]\big)^{T}$};
     \draw[-,thick]    (x1) -- (x2);
       \draw[-,thick]    (x1) -- (x3);
\end{tikzpicture}
\caption{Example of a CIG underlying vector-valued samples $\vx[\timeidx] = (x_1[\timeidx], x_1[\timeidx], x_3[\timeidx])^T$, 
for $\timeidx=1,\ldots,\samplesize$ (see \eqref{equ_proc_model}). 
}
\label{fig_process_model}
\vspace*{-5mm}
\end{figure}

Since we model the observed samples $\vx[\timevar]$ as realizations of a Gaussian process (see \eqref{equ_proc_model}), 
the edges of the CIG can be read off conveniently from the inverse covariance (precision) matrices  
$\mK[\timevar] \defeq \big(\mC[\timevar]\big)^{-1}$ (see \eqref{equ_def_cov_matrix}). In particular, $\component{i}$ are 
$\component{j}$ are conditionally independent, 
given $\{ \component{r} \}_{r \in \mathcal{V} \setminus \{i,j\}}$, 
if and only if $K_{i,j}[\timevar] =0$ for all $\timevar \in \{1,\ldots,\samplesize\}$ \cite[Prop.\ 1.6.6]{Brockwell91}.
Thus, we have the following characterization of the CIG $\cig$ 
associated with the process $\vx[\timeidx]$: 
\begin{equation}
\label{equ_edge_absent_corr_operator}
 \{ i,j \} \!\notin\! \edges  \mbox{ if and only if } K_{i,j}[\timevar]\!=\!0 \mbox{ for } \timevar=1,\ldots,\samplesize.
\end{equation} 
Note that the CIG characterization \eqref{equ_edge_absent_corr_operator} involves a coupling over all 
samples $\{ \vx[\timevar] \}_{\timevar \in \{1,\ldots,\samplesize\}}$. Indeed, an edge is absent between 
two different nodes $i,j \in \nodes$ in the CIG, i.e., $\{i,j\} \notin \edges$, if and only if the precision matrix 
entry $ K_{i,j}[\timevar]$ is zero \emph{for all} $\timeidx \in \{1,\ldots,\samplesize\}$. 

The strength of a connection $\{i,j\} \in \edges$ between process components $\component{i}$ 
and $\component{j}$ is measured by the \emph{average connection strength}
\vspace*{-2mm}
\begin{align}
\rho^{2}_{i,j}  & \defeq 
(1/\samplesize) \sum_{\timeidx=1}^{\samplesize} K^{2}_{i,j}[\timeidx]  /  K_{i,i}^2[\timeidx]  \nonumber \\
& \stackrel{\eqref{equ_cov_matrix_blocks}}{=} (1/\nrblocks)   \sum_{\blkidx=1}^{\nrblocks} \big( K^{(\blkidx)}_{i,j}\big)^{2}  /  \big( K^{(\blkidx)}_{i,i}\big)^{2}.
\label{equ_partial_correlation_def}
\vspace*{-2mm}
\end{align}
We highlight that the quantity $\rho^{2}_{i,j}$ is determined by the conditional distribution of $\vx_{i}$ 
and $\vx_{j}$ given all the remaining process components $\vx_{m}$ with $m \in \{1,\ldots,\coeffdim\} \setminus \{i,j\}$. 
The average connection strength $\rho^{2}_{i,j}$ is closely related to the average squared 
partial correlations (conditional correlation coefficients) \cite{MeinBuhl2006,WangWain2010,RavWainRaskYu2011}
\begin{equation}
\nonumber
(1/\nrblocks)\sum_{\blkidx=1}^{\nrblocks}  \big( K^{(\blkidx)}_{i,j}[\timeidx] \big)^{2}  /  (K^{(\blkidx)}_{i,i}[\timeidx] K^{(\blkidx)}_{j,j}[\timeidx]).
\end{equation}
In contrast to the squared partial correlation, the connection strength \eqref{equ_partial_correlation_def} 
is not symmetric since $\rho^{2}_{i,j} \neq \rho^{2}_{j,i}$ in general. 
However, we find the non-symmetric notion of connection strength more natural for our analysis of the 
simple GMS method proposed in Section \ref{sec_GMS_neighbrohoud_regr}. 

By \eqref{equ_edge_absent_corr_operator} and \eqref{equ_partial_correlation_def}, two nodes 
$i, j \in \nodes$ are connected by an edge $\{i,j\}\!\in\!\edges$ if and only if $\rho_{i,j}\!\neq\!0$. 
Note that $\rho_{i,j}$ is an average measure, i.e., even if the marginal connection strength 
$K^{2}_{i,j}[\timeidx]  /  K_{i,i}^2[\timeidx]$ is very small for some $n$, the average connection 
strength $\rho^{2}_{i,j}$ might still be sufficiently large. The definition \eqref{equ_partial_correlation_def} is a natural 
extension of (a non-symmetric version of) \cite[Eq. (3)]{WangWain2010}, which considers i.i.d.\ samples, 
to the non-stationary model \eqref{equ_proc_model} considered in this paper. 

Accurate estimation of the CIG $\cig$ based on a finite number $\samplesize$ of samples (incurring unavoidable estimation errors) 
is only possible for sufficiently large connection strength $\rho^{2}_{i,j}$ for all edges $\{i,j\} \in \edges$ in the CIG $\cig$. 
\vspace*{-1mm}
\begin{assumption} 
\label{aspt_minimum_par_cor}
The average connection strength $\rho_{i,j}^{2}$ (see \eqref{equ_partial_correlation_def}) between any two connected 
components $\vx_{i}$ and $\vx_{j}$ with $\{i,j\} \in \edges$ is lower bounded 
\vspace*{-2mm}
\begin{equation} 
\rho^{2}_{i,j} \geq \rho^{2}_{\rm min} \hspace*{2mm} \mbox{ for every } \{i,j\} \in \edges, 
\label{equ_rho_min_aspt}
\vspace*{-2mm}
\end{equation}
with some known lower bound $\rho^{2}_{\rm min}>0$.
\end{assumption}

Given a node $i \in \nodes$ in the CIG $\cig$, we denote its neighbourhood and degree as $\mathcal{N}(i) \defeq \{ j \in \nodes \setminus \{i\} : \{ i,j\} \in \edges\}$ 
and $s_{i} = |\mathcal{N}(i)|$, respectively. While in principle, our analysis of GMS applies to processes with arbitrary CIG structure, 
our results will be most useful if the underlying CIG $\cig$ is sparse in the sense of having a small (bounded) maximum node degree. 
\vspace*{-1mm}
\begin{assumption} 
	\label{aspt_cig_sparse}
	The node degrees in the CIG are bounded by some sparsity level $\sparsity$ as  
	\vspace*{-2mm}
	\begin{equation}
	\max_{i \in \nodes} s_{i} \leq \sparsity \mbox{, with  } \sparsity\!<\! (\procdim/3) \!\wedge\! (\blockLength/3).  
	\label{equ_sparsity_aspt}
	\end{equation}
\end{assumption}
We highlight that our approach to GMS for the process \eqref{equ_proc_model} requires a known sparsity level $\sparsity$ 
for the upper bound \eqref{equ_sparsity_aspt}. However, in contrast to \cite{Wain2009TIT}, the sparsity $\sparsity$ is only 
required to form an upper bound on the node degrees $s_{i}$. In particular, our approach is able to handle nodes $i \in \cig$ 
which have smaller degrees $s_{i} < \sparsity$. 

The requirement \eqref{equ_sparsity_aspt} implies a trade-off between the block length $\blocklen$ of consecutive i.i.d.\ samples 
in \eqref{equ_proc_model} and the sparsity $\sparsity$ of the underlying CIG. In particular, for a given sample size $\samplesize$, 
we can tolerate less smoothness (smaller block length $\blockLength$ in \eqref{equ_proc_model}), if the underlying CIG is 
more sparse (having smaller maximum degree $\sparsity$).

It will be notationally convenient to assume the samples $\vx[\timeidx]$ suitably scaled such that 
the eigenvalues of the covariance matrices $\CMX{\timevar}$ are bounded with known constants. 
\vspace*{-1mm}
\begin{assumption} 
\label{aspt_eig_val}
The eigenvalues of the covariance matrices $\CMX{\timeidx}$ are bounded as 
\vspace*{-2mm}
\begin{equation} 
\label{equ_bounds_eigvals_asspt3}
1\!\leq\! \lambda_{l}(\CMX{\timeidx})\!\leq\!  \beta \mbox{ for all } l \in \{1,\ldots,\coeffdim\} \mbox{ and } \timeidx \in \{1,\ldots,\samplesize\},
\end{equation}
with some known upper bound $\beta\!\geq\!1$.
\end{assumption}
Fixing the lower bound in Assumption \ref{aspt_eig_val} to be equal to $1$ is not restrictive since we assume the covariance 
matrices $\CMX{\timeidx}$ to be invertible. 

\section{Sparse Neighborhood Regression}
\label{sec_GMS_neighbrohoud_regr}

The CIG $\cig$ of the process $\vx[\timevar]$ in \eqref{equ_proc_model} is fully specified by the neighbourhoods of the 
nodes in the CIG. Indeed, rather trivially, we can determine the CIG by determining the neighbourhoods $\mathcal{N}(i)$ 
separately for each node $i \in \nodes$. Thus, without loss of generality, we will focus on the sub-problem of determining 
the neighbourhood $\mathcal{N}(i)$ of an arbitrary but fixed node $i \in \nodes$.

In view of the process model \eqref{equ_proc_model} we define, for an arbitrary but fixed block $\blkidx \!\in\! \{1,\ldots,\nrblocks\}$, 
the $i$th process component as
\begin{equation*}
\vx_i^\blknot \defeq \big(x_{i}[(\blkidx-1)\blockLength+1],\ldots,x_{i}[\blkidx \blockLength]\big)^{T} \in \mathbb{R}^{\blockLength}. 
\end{equation*}  
The process components of different blocks are uncorrelated, i.e., 
\begin{equation}
\nonumber
\expect \big\{ \vx_i^{(\blkidx)} \big( \vx_{j}^{(\blkidx')} \big)^{T} \big\} = \mathbf{0} \mbox{ for } \blkidx \neq \blkidx' \mbox{ and any } i,j\in \nodes. 
\end{equation} 
Elementary properties of multivariate normal distributions (see, e.g., \cite[Thm. 3.5.1]{Gallager13}) and the fact 
$K_{i,j}[\timevar] = 0$ for $j \notin \mathcal{N}(i)$ (cf. \eqref{equ_edge_absent_corr_operator}), yield
\begin{equation} 
\label{equ_comp_2}
\vx_i^\blknot = \sum_{j \in \mathcal{N}(i)} a_{j} \vx_j^\blknot  + \bfep_{i}^\blknot,
\end{equation} 
with the coefficients $a_j \!=\! -K^{(\blkidx)}_{i,j}/K^{(\blkidx)}_{i,i}$.
The error vector $\bfep_{i}^\blknot \sim  \mathcal{N}(\mathbf{0},(1/K^{(\blkidx)}_{i,i}) \mathbf{I}_{\blockLength})$ 
is uncorrelated with the vectors $\big\{ \vx_j^\blknot \big\}_{j \in \mathcal{N}(i)}$. Note that the random vector 
$\sum_{j \in \mathcal{N}(i)} a_{j} \vx_j^\blknot$ in \eqref{equ_comp_2} is the minimum mean squared error (MMSE) estimator 
of $\vx_i^\blknot$ using the random vectors $\big\{ \vx_j^\blknot \big\}_{j\in \mathcal{N}(i)}$ as observations (see \cite{papoulis}).

Given some index set $\mathcal{T} \subseteq \{1,\ldots,\procdim\}$ with $\mathcal{N}(i)\setminus \setT  \neq \emptyset$, 
another application of \cite[Thm. 3.5.1]{Gallager13} to the component $\sum_{j \in \mathcal{N}(i)} a_{j} \vx_j^\blknot$ 
in the decomposition \eqref{equ_comp_2} yields
\vspace*{-0mm}
\begin{equation}
\label{equ_comp_1}
\vx_i^\blknot = \underbrace{\sum_{j \in \setT} c_{j} \vx_j^\blknot  + \tilde{\vx}_i^\blknot}_{=\sum_{j \in \mathcal{N}(i)} a_{j} \vx_j^\blknot \mbox{ see \eqref{equ_comp_2}}} + \bfep_{i}^\blknot ,
\vspace*{-1mm}
\end{equation} 
with the random vectors $\tilde{\vx}_{i}^\blknot$, $\{ \vx_j^\blknot  \}_{j \in \setT}$ and $\bfep_{i}^\blknot$ being jointly Gaussian. 
Moreover, the random vectors $\tilde{\vx}_{i}^\blknot$ are uncorrelated with the random vectors $\{ \vx_j^\blknot  \}_{j \in \setT}$, $\bfep_{i}^\blknot$ 
and distributed as 
\begin{equation} 
\label{equ_distribution_tilde_vx}
\tilde{\vx}_{i} ^\blknot\!\sim\!\mathcal{N}(\mathbf{0},\tilde{\sigma}^{2}_{\blkidx} \mathbf{I}_{\blockLength}).
\end{equation} 
Note that the vector $ \tilde{\vx}_i^\blknot$ in \eqref{equ_comp_1} is the estimation error incurred by the MMSE estimator of the 
random vector $\sum_{j \in \mathcal{N}(i)} a_{j} \vx_j^\blknot$ using $\{ \vx_j^\blknot  \}_{j \in \setT}$ as observations. 

Using \cite[Thm. 3.5.1]{Gallager13}, the variance $\tilde{\sigma}^{2}_{\blkidx}$ of (the i.i.d.\ entries of) $\tilde{\vx}_{i}^\blknot$ can be obtained as 
\begin{equation}
\label{equ_tilde_sigma_b} 
\tilde{\sigma}_{\blkidx}^{2} = \mathbf{a}^{T} \big( \widetilde{\mathbf{K}}^{\blknot} \big)^{-1} \mathbf{a}
\end{equation} 
with the matrix $\widetilde{\mathbf{K}}^{\blknot} = \big( \big( \mC^{(\blkidx)}_{\mathcal{N}(i) \cup \mathcal{T}} \big)^{-1} \big)_{\mathcal{N}(i) \setminus \setT}$ 
and the vector $\va \in \mathbb{R}^{|\mathcal{N}(i)\setminus \setT|}$ whose entries are given by $ a_j = -K^{(\blkidx)}_{i,j}/K^{(\blkidx)}_{i,i}$, 
for $j \!\in\! \mathcal{N}(i)\setminus \setT$. In what follows we will make use of a lower bound on the variance $\tilde{\sigma}_{\blkidx}^{2}$ which 
is due to  Assumption \ref{aspt_eig_val}. Indeed, by Assumption \ref{aspt_eig_val} we have $\lambda_{l} \big( \big(\widetilde{\mathbf{K}}^{\blknot}\big)^{-1} \big)\!\geq\!1$, 
for all $l=1,\ldots,|\mathcal{N}(i) \setminus \setT|$, which implies (see \eqref{equ_tilde_sigma_b}) the lower bound     
\vspace*{0mm}
\begin{align}
\tilde{\sigma}^{2}_{\blkidx} & \stackrel{}{\geq}  \hspace*{-2mm} \sum_{j\in \mathcal{N}(i)\setminus \setT} (K_{i,j}^{(\blkidx)}/K_{i,i}^{(\blkidx)})^2. 
\label{equ_events_3_m3_2}
\vspace*{-2mm}
\end{align} 
On the other hand, we can use Assumption \ref{aspt_eig_val} to obtain (via \eqref{equ_comp_1} and \eqref{equ_comp_2}) the upper 
bound\footnote{The variance $\tilde{\sigma}^{2}_{\blkidx}$ of (the i.i.d.\ entries of) the random vector $\tilde{\vx}^\blknot_{i}$ does not 
exceed the variance of (the i.i.d.\ entries of) the random vector $\sum_{j \in \mathcal{N}(i)} a_{j} \vx_j^\blknot$ due to the (orthogonal) 
decomposition \eqref{equ_comp_1}. The variance of $\sum_{j \in \mathcal{N}(i)} a_{j} \vx_j^\blknot$ is, in turn, upper bounded by the 
variance of the random vector $\vx_i^\blknot$  due to the (orthogonal) decomposition \eqref{equ_comp_2}. } 
\begin{equation} 
\label{equ_bound_simga_2_blokdx}
\tilde{\sigma}^{2}_{\blkidx} \leq \beta. 
\end{equation} 
It will be convenient to stack the vectors $\tilde{\vx}_{i} ^\blknot$ (cf. \eqref{equ_distribution_tilde_vx}) into a single Gaussian random vector 
\begin{align}
\label{equ_def_vx_i_long}
\tilde{\vx}_{i} & \defeq \big( \big( \tilde{\vx}_{i}^{(1)} \big)^{T},\ldots, \big( \tilde{\vx}_{i}^{(\nrblocks)} \big)^{T} \big) \sim \mathcal{N}(\mathbf{0},\mathbf{C}_{\tilde{x}_{i}}) \nonumber \\
& \mbox{with } \mathbf{C}_{\tilde{x}_{i}} = 
{\rm blkdiag} \{\tilde{\sigma}^{2}_{\blkidx} \mathbf{I}_{\blockLength}\}_{\blkidx=1}^{\nrblocks}.
\end{align}

The decompositions \eqref{equ_comp_2} and \eqref{equ_comp_1} suggest a simple strategy for estimating 
(selecting) the neighbourhoods $\mathcal{N}(i)$ of the nodes $i \in \nodes$ in the CIG $\cig$. To this end, 
let $\mP_{\setT^{\perp}}^{\blknot} \in \mathbb{R}^{\blocklen \times \blocklen}$ 
denote the orthogonal projection matrix for the complement of the subspace 
$\mathcal{X}^{\blknot}_{\set T} \defeq {\rm span} \big\{ \vx_{j}^{\blknot} \big\}_{j \in \setT} \subseteq \mathbb{R}^{\blocklen}$, i.e., 
\begin{equation} 
\label{equ_def_P_b_setT}
\mP_{\setT^{\perp}}^{(\blkidx)} \defeq \mathbf{I}-  \mP_{\setT}^{(\blkidx)}  \mbox{, with } \mP_{\setT}^{(\blkidx)} 
\defeq \sum_{j=1}^{{\rm dim} \mathcal{X}^{\blknot}_{\setT}} \mathbf{u}^{\blknot}_{j} \big( \mathbf{u}^{\blknot}_{j}\big)^{T}, 
\end{equation}
with $\big\{  \mathbf{u}^{\blknot}_{j} \big\}_{j=1}^{{\rm dim} \mathcal{X}_{\setT}}$ being an orthonormal basis for the 
subspace $\mathcal{X}^{\blknot}_{\setT} \subseteq \mathbb{R}^{\blocklen}$. The matrix $\mP_{\setT}^{(\blkidx)}$ 
in \eqref{equ_def_P_b_setT} is an orthogonal projection matrix on the subspace $\mathcal{X}^{\blknot}_{\setT}$. 

According to \eqref{equ_comp_2}, for any index set $\setT \supseteq \mathcal{N}(i)$ 
(such that $\mathcal{N}(i) \setminus \mathcal{T}\!=\!\emptyset$), 
\begin{equation} 
\label{equ_overlap_superset_power}
 \|\mP_{\setT^{\perp}}^\blknot \vx_i^\blknot \|_2^2 =  \|\mP_{\setT^{\perp}}^\blknot  \bfep_{i}^\blknot \|_2^2 
 \mbox{ for all } \blkidx \in \{1,\ldots,\nrblocks\}.
\end{equation} 
On the other hand, for any index set $\mathcal{T}$ with $\mathcal{N}(i) \setminus \mathcal{T} \neq \emptyset$, \eqref{equ_comp_1} entails 
\begin{equation} 
\label{equ_overlap_non_empty_power}
 \|\mP_{\setT^{\perp}}^\blknot \vx_i^\blknot \|_2^2 =   \| \mP_{\setT^{\perp}}^\blknot  
 (\tilde{\vx}_i^\blknot + \bfep_{i}^\blknot) \|_2^2 \mbox{ for all } \blkidx \in \{1,\ldots,\nrblocks\}, 
\end{equation} 
with some random vector $\tilde{\vx}_i^\blknot$. 
Some of our efforts go into showing that 
\begin{equation} 
\label{equ_approx_ident_mP} \nonumber
\| \mP_{\setT^{\perp}}^\blknot  (\tilde{\vx}_i^\blknot + \bfep_{i}^\blknot) \|_2^2 \approx  \| \mP_{\setT^{\perp}}^\blknot  \tilde{\vx}_i^\blknot  \|_2^2 
+ \| \mP_{\setT^{\perp}}^\blknot   \bfep_{i}^\blknot  \|_2^2,  
\end{equation} 
for all $\blkidx \in \{1,\ldots,\nrblocks\}$. 
Thus, according to \eqref{equ_overlap_superset_power} and \eqref{equ_overlap_non_empty_power}, if the component 
$\tilde{\vx}_i^\blknot$ in \eqref{equ_comp_1} is not too small, the estimator 
\vspace*{-2mm}
\begin{align}
\label{equ_neigborhood_est}
\hspace*{-5mm} \widehat{\mathcal{N}}(i) & \!\defeq\! 
\argmin_{|\setT| \leq \sparsity}  (1/\samplesize) \sum\limits_{\blkidx=1}^{\nrblocks} \|\mP_{\setT^{\perp}}^{\blknot} \vx_i^{\blknot}\|_2^2 + \lambda | \setT|,
\vspace*{-2mm}
\end{align} 
delivers the true neighbourhood, i.e., $\widehat{\mathcal{N}}(i)=\mathcal{N}(i)$, with high probability. 
The penalty term $\lambda | \setT|$ in \eqref{equ_neigborhood_est} is required since we allow nodes $i \in \nodes$ 
in the CIG to potentially have fewer than $\sparsity$ neighbours ($|\mathcal{N}(i)| < \sparsity$).\footnote{
In contrast to our approach \eqref{equ_neigborhood_est}, the analysis of GMS presented in \cite{Wain2009TIT}, 
for the special case of i.i.d.\ samples, requires all neighbourhoods $\mathcal{N}(i)$ to have exactly the same size $\sparsity$, 
i.e., $\mathcal{N}(i)\!=\!\sparsity$ for all $i\!\in\!\nodes$.} 
Indeed, the statistic $\|\mP_{\setT^{\perp}}^{\blknot} \vx_i^{\blknot}\|_2^2$ does not allow to distinguish 
between different sets $\setT$ which contain the neighborhood $\mathcal{N}(i)$. Therefore, we need to add the 
penalty term $\lambda | \setT|$ in \eqref{equ_neigborhood_est} in order to prefer smaller sets $\setT$ as an estimate 
for $\mathcal{N}(i)$. 

The estimator \eqref{equ_neigborhood_est} performs sparse block-wise least squares regression by approximating the 
$i$th component $\vx_{i}$ (cf.\ \eqref{equ_component_i}) in a sparse manner (by allowing only $\sparsity$ active components) 
using the remaining process components. Indeed, the summands $\|\mP_{\setT^{\perp}}^{\blknot} \vx_i^{\blknot}\|_2^2$ 
in \eqref{equ_neigborhood_est} are the errors obtained from the block-wise regression problems
\begin{equation} 
\nonumber 
 \big\|\mP_{\setT^{\perp}}^{\blknot} \vx_i^{\blknot} \big\|_2^2 
 = \min_{\vw^{\blknot}\!\in\!\mathbb{R}^{\coeffdim}, w^{\blknot}_{i}\!=\!0} \big\| \vx_{i}^{\blknot} - \sum_{j \in \setT} w^{\blknot}_{j} \vx_{j}^{\blknot} \big\|^{2}_{2}. 
\end{equation} 
 
We highlight that the estimator \eqref{equ_neigborhood_est} is mainly useful as a theoretical device which allows for a 
simple performance analysis and, in turn, a characterization of the required sample size for accurate GMS. Using  
a naive implementation of \eqref{equ_neigborhood_est}, by searching over all subsets of $\{1,\ldots,\coefflen\}$ with size 
at most $\sparsity$, has a complexity which grows exponentially in the sparsity level $\sparsity$. Thus, the estimator 
\eqref{equ_neigborhood_est} is typically intractable except for very small sparsity levels $\sparsity$ (corresponding to 
a very sparse CIG). More tractable methods for GMS can be obtained by using convex relaxations of \eqref{equ_neigborhood_est} 
which result in Lasso-type methods (see \cite{DanaherGroupGLASSO2014} and Section \ref{sec_chain}).

The estimator \eqref{equ_neigborhood_est} itself only delivers an estimate for the neighbourhood of some node $i \in \nodes$ in the 
CIG $\cig$ underlying the process \eqref{equ_proc_model}. In order to obtain an estimate of the entire CIG, 
we have to repeatedly apply the estimator \eqref{equ_neigborhood_est} to each node $i \in \nodes$. It might then happen that 
due to estimation errors, we obtain $i \in \widehat{\mathcal{N}}(j)$ but $j \notin \widehat{\mathcal{N}}(i)$ for two different 
nodes $i,j \in \nodes$. There are different options how to handle such a situation such as insisting in consistency between 
the neighbourhoods when declaring the presence of an edge (see \cite[Eq. (7)]{MeinBuhl2006}). 
However, the implementation details for handling such cases are not relevant to our analysis, 
which aims at sufficient conditions such that \eqref{equ_neigborhood_est} delivers the correct neighborhood 
for all nodes simultaneously (with high probability). 

Our main result is an upper bound on the probability of the sparse neighbourhood regression \eqref{equ_neigborhood_est} 
to fail in delivering the correct neighbourhood $\mathcal{N}(i)$, i.e., the error event
\begin{equation} 
\label{equ_def_error_no_superset}
\error \defeq \{ \mathcal{N}(i) \neq \widehat{\mathcal{N}}(i) \}.
\end{equation}  
\begin{theorem}
\label{thm_upper_bound_err_prob}
Consider the vector samples $\vx[\timeidx]$, for $\timeidx=1,\ldots,\samplesize$, conforming to the process model 
\eqref{equ_proc_model} and such that Assumption \ref{aspt_minimum_par_cor}, \ref{aspt_cig_sparse} and \ref{aspt_eig_val} 
are valid. We estimate the neighbourhood $\mathcal{N}(i)$ of an arbitrary but fixed node $i \in \nodes$ in the CIG via 
sparse regression \eqref{equ_neigborhood_est} with $\lambda = \rho^{2}_{\rm min}/6$. Then, if the average connection strength 
between connected components are sufficiently large such that (see \eqref{equ_rho_min_aspt})
\vspace*{1mm}
\begin{equation}
\label{equ_condition_rho_min_kappa}
\rho^{2}_{\rm min} \geq 24 \beta / \blocklen
\vspace*{1mm}
\end{equation} 
for any sample size 
\vspace*{1mm}
\begin{equation}
\label{equ_suff_cond_ell_3_thm}
\samplesize \!\geq\! 864   (\beta/\rho^{2}_{\rm min}) \log( 6 \coeffdim  \sparsity^2 /\eta),
\vspace*{2mm}
\end{equation} 
the probability of the error event \eqref{equ_def_error_no_superset} is bounded as $\prob \{ \error \} \leq \eta$. 
\end{theorem} 
By Theorem \ref{thm_upper_bound_err_prob}, the true neighbourhood $\mathcal{N}(i)$ of a node $i \in \nodes$ 
can be recovered via \eqref{equ_neigborhood_est} with high probability if the samples size $\samplesize$ is on 
the order of $(\beta/\rho^{2}_{\rm min}) \log(\procdim \sparsity^{2})$ (for a fixed error tolerance $\eta$). Therefore, 
given sufficiently large computational power, GMS via sparse neighbourhood regression \eqref{equ_neigborhood_est} 
is feasible in the high dimensional regime where $\samplesize \ll \procdim$. 

Since a CIG $\cig$ is entirely determined by the neighbourhoods $\mathcal{N}(i)$ of all nodes $i\in \nodes$, 
we obtain the following result on GMS as a direct consequence of Theorem \ref{thm_upper_bound_err_prob}. 
\begin{corollary}
\label{cor_upper_bound_err_prob}
Consider a process \eqref{equ_proc_model} with underlying CIG $\cig$ and satisfying all the assumptions 
in Theorem \ref{thm_upper_bound_err_prob}. 
Then, for any sample size 
\vspace*{1mm}
\begin{equation}
\label{equ_suff_cond_ell_3}
\samplesize \!\geq\! 864   (\beta/\rho^{2}_{\rm min}) \log( 6 \coeffdim^{2}  \sparsity^2 /\eta),
\vspace*{2mm}
\end{equation} 
there is a GMS method delivering a CIG estimate $\widehat{\cig}$ with $\prob \big\{ \widehat{\cig} \neq \cig  \big\} \leq \eta$. 
\end{corollary} 
\begin{proof}
Using \eqref{equ_neigborhood_est}, we compute an estimate $\widehat{\mathcal{N}}(i)$ for each node $i \in \nodes$. 
Then, we construct a CIG estimate $\widehat{\cig}$ having an edge $\{i,j\}$ between nodes $i,j \in \nodes$ when $j\!\in\!\widehat{N}(i)$ 
and $i \in \mathcal{N}(j)$. The estimate $\widehat{\cig}$ is correct, i.e., $\widehat{\cig} = \cig$ whenever all of the estimates 
$\widehat{\mathcal{N}}(i)$ are correct, i.e., $\widehat{\mathcal{N}}(i) = \mathcal{N}(i)$ for each node $i \in \nodes$. The result 
then follows by combining Theorem \ref{thm_upper_bound_err_prob} with a union bound (over all nodes $i \in \nodes$). 
\end{proof} 
The bound \eqref{equ_suff_cond_ell_3} indicates that accurate GMS (with prescribed small error rate $\eta$) from time series 
data conforming to the model \eqref{equ_proc_model} is possible for a sample size $\samplesize \propto (1/\rho^{2}_{\rm min}) \log (\coeffdim \sparsity)$. 
We note that the bound \eqref{equ_suff_cond_ell_3} improves existing bounds on the sample size required for particular GMS methods based on 
convex optimization \cite{LeeLiu2015}. In particular, while the results in \cite{LeeLiu2015} indicate that the 
sample size $\samplesize$ required for GMS scales with $\log \coeffdim$, they do not provide the explicit 
dependence of $\samplesize$ on the guaranteed error rate $\prob \big\{ \widehat{\cig} \neq \cig  \big\}$, sparsity $\sparsity$ 
(see \eqref{equ_sparsity_aspt}) and minimum connection strength $\rho^{2}_{\rm min}$ (see \eqref{equ_rho_min_aspt}).

It turns out that the bound \eqref{equ_suff_cond_ell_3} is sharp since it matches a fundamental lower bound on the required 
sample size for any GMS method which performs uniformly well for any process of the form \eqref{equ_proc_model} and 
satisfying Assumption \ref{aspt_minimum_par_cor}, \eqref{aspt_cig_sparse} and \ref{aspt_eig_val}. This lower bound follows 
directly from the results in \cite{HannakJung2014conf}. 
\begin{lemma}{\cite[Theorem 3.1]{HannakJung2014conf}}
\label{lem_lower_bound}
Consider a GMS method which reads in vector samples $\vx[1],\ldots,\vx[\samplesize]$ (see \eqref{equ_proc_model}) and 
delivers an estimate $\widehat{\cig}$ for the CIG $\cig$ between the components $\component{i}$, for $i=1,\ldots,\featurelen$ 
(see \eqref{equ_component_i}). If the method achieves an error probability $\prob \{ \error \}$ uniformly bounded by some 
prescribed error level $\eta$ for any process of the form \eqref{equ_proc_model} satisfying Assumption \ref{aspt_minimum_par_cor}, 
\ref{aspt_cig_sparse} and \ref{aspt_eig_val} with $\rho^{2}_{\rm min} \leq 1/4$, then the sample size must necessarily satisfy 
$\samplesize > \frac{ \log \genfrac(){0pt}{2}{\featurelen}{2} -1}{4 \rho_{\rm min}^{2}}$. 
\end{lemma} 
Combining Theorem \ref{thm_upper_bound_err_prob} with Lemma \ref{lem_lower_bound}, we conclude that the bound 
\eqref{equ_suff_cond_ell_3} characterizes, up to a constant factor, the minimum required sample size for accurate GMS 
based on processes of the form \eqref{equ_proc_model}. 

It is instructive to compare the sufficient condition \eqref{equ_suff_cond_ell_3} on the sample size $\samplesize$ with the 
results obtained in \cite{RavWainRaskYu2011,Wain2009TIT,WangWain2010} for the special case of i.i.d.\ samples, which 
coincides with the model \eqref{equ_proc_model} for $\nrblocks\!=\!1$ and $\samplesize\!=\!\blocklen$. We note that for this 
special case, the bound \eqref{equ_suff_cond_ell_3} matches the necessary condition on sample size derived in \cite{WangWain2010}, 
which confirms the sparse regression method \eqref{equ_neigborhood_est} to be optimal in terms of sample size requirement. 
However, this is already certified by Lemma \ref{lem_lower_bound}, which is extends the results of \cite{WangWain2010} 
to non-stationary processes \eqref{equ_proc_model}. 

At first sight it appears that the bound \eqref{equ_suff_cond_ell_3} suggests a smaller required sample size compared to the bound 
$\samplesize \!\propto\! \sparsity^{2} \log \featurelen$ obtained in \cite[Corollary 1]{RavWainRaskYu2011}. 
However, it is important to note that the lower bound $\rho^{2}_{\rm min}$ (see \eqref{equ_rho_min_aspt}) on the minimum connection strength 
$\rho^{2}_{i,j}$ (between connected components) cannot be chosen arbitrarily in order to have at least one 
process \eqref{equ_proc_model} satisfying Assumption \ref{aspt_minimum_par_cor}. In particular, the off-diagonal entries $K_{i,j}[\timeidx]$ 
of the precision matrices cannot take on arbitrary (large) values, since the precision matrix $\mathbf{K}[\timeidx]\!=\!\big(\mC[\timeidx]\big)^{-1}$ 
(see \eqref{equ_def_cov_matrix}) must be positive definite. 

A practically relevant regime for the minimum connection strength is $\rho^{2}_{\rm min} \leq c/\sparsity^2$ with some constant $c$ which 
may depend on $\beta$ (see \eqref{equ_bounds_eigvals_asspt3}). For this regime, which is also considered in \cite{WangWain2010}, 
the bound \eqref{equ_suff_cond_ell_3} becomes $\samplesize \propto \sparsity^{2} \log \featurelen$ which closely resembles the sample 
size requirement for the convex GMS method in \cite{RavWainRaskYu2011}. 

Finally, we note that Theorem \ref{thm_upper_bound_err_prob} does not involve some incoherence condition, which requires sub-matrices 
of the covariance matrices $\mC^{(\blkidx)}$ (see \eqref{equ_cov_matrix_blocks}) to be well-conditioned. Such incoherence conditions 
are typically required by convex relaxations of the sparse regression estimator \eqref{equ_neigborhood_est}. While convex (Lasso-based) methods 
are computationally more tractable than non-convex estimators such as \eqref{equ_neigborhood_est}, convex methods place more stringent 
conditions (such as some incoherence condition) on the process \eqref{equ_proc_model} in order to guarantee accurate estimation of the 
underlying CIG \cite{Wain2009TITSharpThresholds,Loh2017}.

%


\vspace*{-0mm}
\section{Numerical Results} 
\vspace*{-0mm}
\label{sec_numerical_results}
We verify the predictions of Theorem \ref{thm_upper_bound_err_prob} by means of numerical experiments involving 
synthetic data (see Section \ref{sec_chain}) and data collected by pedestrian count devices located in the city of 
Turku in Finland (see Section \ref{ped_counts}). We also compare our results with the empirical performance obtained 
from computationally efficient convex optimization methods (see Section \ref{sec_chain}). In order to support reproducible 
research, we have made the source code for our experiments available under \url{https://github.com/alexjungaalto/ResearchPublic/tree/master/GMSNonStat}. 


\subsection{Chain}
\label{sec_chain}
Our first experiment revolves around a synthetic dataset $\vx[\timeidx]$ which is generated according to the process model 
\eqref{equ_proc_model} such that the true underlying CIG $\cig$ is a chain graph as depicted in Figure \ref{chainGraph}.

\begin{figure}[ht]
\centering
\begin{tikzpicture}[x=1cm,y=1cm]
 \node[draw,circle,thick] (x1) at (1,0.6) {$\vx_1$};
  \node[draw,circle,thick] (x2) at (3,0.6) {$\vx_2$};
   \node[draw,circle,thick] (x3) at (5,0.6) {$\vx_3$};
    \node[draw,circle,thick] (x4) at (7,0.6) {$\vx_4$};
        \node[draw,circle,thick] (xp) at (9,0.6) {$\vx_{\procdim}$};
     \draw[-,thick]    (x1) -- (x2);
       \draw[-,thick]    (x2) -- (x3);
       \draw[-,thick]    (x3) -- (x4);
         \draw[-,thick,dashed]    (x4) -- (xp);  
\end{tikzpicture}
\caption{The CIG of a process \eqref{equ_proc_model} with a chain structure.}
\label{chainGraph}
\end{figure}

In particular, we generated Gaussian random vectors conforming to the process model \eqref{equ_proc_model} 
with $\nrblocks=4$ blocks. The $\blkidx$-th block consists of $L$ i.i.d.\ random vectors $\vx[\timeidx]\!\sim\!\mathcal{N}(\mathbf{0},\mathbf{C}^{(\blkidx)})$ 
with $\mathbf{C}^{(\blkidx)}$ being chosen such that the marginal CIG $\cig^{(\blkidx)}$ is a chain (see Figure \ref{equ_proc_model}) 
with the edge $\{\blkidx,\blkidx+1\}$ missing (see Figure \ref{chainGraph_block}).  

\begin{figure}[ht]
\centering
\begin{tikzpicture}[x=0.8cm,y=1cm]
\node[draw,circle,thick] (x1) at (1,0.5) {$\vx^{(1)}_1$};
\node[draw,circle,thick] (x2) at (3,0.5) {$\vx^{(1)}_2$};
\node[draw,circle,thick] (x3) at (5,0.5) {$\vx^{(1)}_3$};
\node[draw,circle,thick] (x4) at (7,0.5) {$\vx^{(1)}_4$};
\node[draw,circle,thick] (xp) at (9,0.5) {$\vx^{(1)}_{\procdim}$};
\node[draw,circle,thick] (x12) at (1,-1.5) {$\vx^{(2)}_1$};
\node[draw,circle,thick] (x22) at (3,-1.5) {$\vx^{(2)}_2$};
\node[draw,circle,thick] (x32) at (5,-1.5) {$\vx^{(2)}_3$};
\node[draw,circle,thick] (x42) at (7,-1.5) {$\vx^{(2)}_4$};
\node[draw,circle,thick] (xp2) at (9,-1.5) {$\vx^{(2)}_{\procdim}$};
\node[] at (5,-2.5) {$\vdots$};
     \draw[-,thick]    (x2) -- (x3);
       \draw[-,thick]    (x3) -- (x4);
         \draw[-,thick,dashed]    (x4) -- (xp);  
         
      \draw[-,thick]    (x12) -- (x22);
       \draw[-,thick]    (x32) -- (x42);
         \draw[-,thick,dashed]    (x42) -- (xp2);  
\end{tikzpicture}
\caption{The marginal CIG $\cig^{(\blkidx)}$ underlying the $\blkidx$-th block, constituted by the i.i.d. 
samples $\vx[(\blkidx\!-\!1)\blocklen\!+\!1],\ldots,\vx[\blkidx\blocklen]$, is a chain with the edge $\{\blkidx,\blkidx+1\}$ removed.}
\label{chainGraph_block}
\end{figure}
  

In order to estimate the neighbourhood $\mathcal{N}(i)$ of a given node $i \in \nodes$ in the CIG from the generated vector samples, 
we use the sparse regression estimator \eqref{equ_neigborhood_est} with $\sparsity=2$. For such a small sparsity, It is still feasible to 
compute the estimator \eqref{equ_neigborhood_est} by exhaustive search over all subsets of size at most $\sparsity=2$. However, for 
larger sparsity $\sparsity$, \eqref{equ_neigborhood_est} becomes intractable and one has to use computationally cheaper 
methods such as convex optimization methods \cite{DanaherGroupGLASSO2014,DistrOptStatistLearningADMM}. 
 
We estimate the error probability \eqref{equ_def_error_no_superset} by an empirical average $\widehat{\prob}\{\error\}$ over $K=100$ 
i.i.d.\ simulation runs. In particular, using the $j$-th realization of the process \eqref{equ_proc_model} as input to the sparse regression 
estimator \eqref{equ_neigborhood_est}, yielding the estimate $\widehat{\mathcal{N}}^{(j)}(i)$, we compute the empirical error rate 
\begin{equation}
\label{equ_def_error_rate}
\widehat {\prob} \{ \error \} \defeq (1/K) \sum_{j=1}^{K} \mathcal{I}( \widehat{\mathcal{N}}^{(j)}(i) \neq \mathcal{N}(i)) (\approx \prob \{ \error \}). 
\end{equation}
Here, $\hat{\mathcal{N}}^{(j)}(i)$ is the estimated neighbourhood of node $i\!\in\!\nodes$ during 
the $j$-th simulation run. 


\begin{figure}[htbp]
\begin{minipage}[t]{0.7\columnwidth}
\begin{tikzpicture}
 \tikzset{x=5cm,y=3cm,every path/.style={>=latex},node style/.style={circle,draw}}
      \node [anchor=south] at (0,1.1) {$P_{\rm err}$};
           \foreach \label/\labelval in {0/$0$,0.25/$0.25$,0.5/$0.5$,0.75/$0.75$,1/$1$}
        { 
          \draw (1pt,\label) -- (-1pt,\label) node[left] {\labelval};
        }       
             \foreach \label/\labelval in {0/$0$,0.25/$0.25$,0.5/$0.5$,0.75/$0.75$,1/$1$} 
        { 
          \draw (\label,1pt) -- (\label,-2pt) node[below] {\labelval};
        }
          \draw[->] (0,0) -- (1.2,0);
      \node [] at (0.5,-0.3) {$\samplesize/500$};
      \draw[->] (0,0) -- (0,1.1);
      
    \csvreader[ head to column names,%
                late after head=\xdef\aold{\x}\xdef\bold{\y},%
                after line=\xdef\aold{\x}\xdef\bold{\y}]%
                {RawSample_p01_task00.csv}{}  
                {\draw [line width=0.2mm] (\aold/500, \bold) -- (\x/500,\y) node {\large $+$};
    }
        \csvreader[ head to column names,%
                late after head=\xdef\aold{\x}\xdef\bold{\y},%
                after line=\xdef\aold{\x}\xdef\bold{\y}]%
                {RawSample_p02_task00.csv}{}  
                {\draw [line width=0.2mm] (\aold/400, \bold) -- (\x/400,\y) node {\large $\times$};
    }
        \csvreader[ head to column names,%
                late after head=\xdef\aold{\x}\xdef\bold{\y},%
                after line=\xdef\aold{\x}\xdef\bold{\y}]%
                {RawSample_p03_task00.csv}{}  
                {\draw [line width=0.2mm] (\aold/400, \bold) -- (\x/400,\y) node {\large $\circ$};
    }
        \csvreader[ head to column names,%
                late after head=\xdef\aold{\x}\xdef\bold{\y},%
                after line=\xdef\aold{\x}\xdef\bold{\y}]%
                {RawSample_p04_task00.csv}{}  
                {\draw [line width=0.2mm] (\aold/400, \bold) -- (\x/400,\y) node {\large $\star$};
    }
    \node [] at (0.5,-0.7) {(a)};
        \vspace*{-3mm}
\end{tikzpicture}
\end{minipage}

\begin{minipage}[t]{0.7\columnwidth}
\begin{tikzpicture}
 \tikzset{x=5cm,y=3cm,every path/.style={>=latex},node style/.style={circle,draw}}
       \node [anchor=south] at (0,1.1) {$P_{\rm err}$};
           \foreach \label/\labelval in {0/$0$,0.25/$0.25$,0.5/$0.5$,0.75/$0.75$,1/$1$}
        { 
          \draw (1pt,\label) -- (-1pt,\label) node[left] {\labelval};
        }       
             \foreach \label/\labelval in {0/$0$,0.25/$0.25$,0.5/$0.5$,0.75/$0.75$,1/$1$} 
        { 
          \draw (\label,1pt) -- (\label,-2pt) node[below] {\labelval};
        }
          \draw[->] (0,0) -- (1.2,0);
      \node [] at (0.5,-0.3) {$\samplesize'/100$};
         \node [] at (0.5,-0.7) {(b)};
      \draw[->] (0,0) -- (0,1.1);
      
    \csvreader[ head to column names,%
                late after head=\xdef\aold{\x}\xdef\bold{\y},%
                after line=\xdef\aold{\x}\xdef\bold{\y}]%
                {ScaledSample_p64_task00.csv}{}
                {\draw [line width=0.2mm] (\aold/100, \bold) -- (\x/100,\y) node {\large $+$};
    }
        \csvreader[ head to column names,%
                late after head=\xdef\aold{\x}\xdef\bold{\y},%
                after line=\xdef\aold{\x}\xdef\bold{\y}]%
                {ScaledSample_p128_task00.csv}{}
                {\draw [line width=0.2mm] (\aold/100, \bold) -- (\x/100,\y) node {\large $\times$};
    }
        \csvreader[ head to column names,%
                late after head=\xdef\aold{\x}\xdef\bold{\y},%
                after line=\xdef\aold{\x}\xdef\bold{\y}]%
                {ScaledSample_p256_task00.csv}{}
                {\draw [line width=0.2mm] (\aold/100, \bold) -- (\x/100,\y) node {\large $\circ$};
    }
            \csvreader[ head to column names,%
                late after head=\xdef\aold{\x}\xdef\bold{\y},%
                after line=\xdef\aold{\x}\xdef\bold{\y}]%
                {ScaledSample_p512_task00.csv}{}
                {\draw [line width=0.2mm] (\aold/100, \bold) -- (\x/100,\y) node {\large $\star$};
    }
         
    \vspace*{-3mm}
\end{tikzpicture}
\end{minipage}
        \vspace*{-3mm}
  \caption{Empirical error rate $P_{\rm err} = \widehat{{\rm P}}\big\{ \widehat{\mathcal{N}}(2) \neq \mathcal{N}(2) \big\}$ 
  (see \eqref{equ_def_error_rate}) incurred by \eqref{equ_neigborhood_est} when estimating the neighbourhood 
  $\mathcal{N}(2)$ of node $i=2$ in a chain CIG with $\coeffdim=64$ (``$+$''), $\coeffdim=128$ (``$\times$''), 
  $\coeffdim=256$ (``$\circ$'') and $\coeffdim=512$ (``$\star$'') nodes. (a) Error rate as a function of the  
  sample size $\samplesize$. (b) Error rate as a function of scaled sample size $\samplesize'=\samplesize \rho_{\rm min}^{2}/ \log \procdim$. 
  Error rate has been obtained using $K=100$ simulation runs.} \label{fig:edgeCase}
\end{figure}


In Figure \ref{fig:edgeCase}-(a) we depict the error rate $\widehat {\prob} \{ \error \}$, achieved by the 
estimator \eqref{equ_proc_model} when estimating the neighbourhood or node $i=2$ (see Figure \ref{chainGraph}), 
as a function of the sample size $\samplesize$. The three curves in Figure \ref{fig:edgeCase}-(a) corresponds 
to three different processes. Each process is of the form \eqref{equ_proc_model} with CIG being a chain 
(see Figure \ref{chainGraph}), but with different $\procdim$ and $\rho_{\rm min}^{2}$ (see \eqref{equ_rho_min_aspt}). 

As indicated by the upper bound \eqref{equ_suff_cond_ell_3_thm} of Theorem \ref{thm_upper_bound_err_prob}, the 
error rate $\widehat {\prob} \{ \error \}$ (see \eqref{equ_def_error_rate}) crucially depends on the scaled sample 
size $\samplesize' \defeq \samplesize \rho_{\rm min}^{2}/ \log \procdim$. Therefore, we plot in Figure \ref{fig:edgeCase}-(b) 
the error rate $\widehat {\prob} \{ \error \}$ as a function of the scaled sample size $\samplesize'$. In agreement with our 
theoretical findings, we observe that the curves in Figure \ref{fig:edgeCase}-(b) are almost lying on top of each other. 

The sparse regression estimator \eqref{equ_neigborhood_est} implements a form of pooling of the samples 
$\vx[1],\ldots,\vx[\samplesize]$ across different blocks. Indeed, the objective function in \eqref{equ_neigborhood_est} 
sums up the contributions from all blocks such that the required sample size depends on the average connection strength 
\eqref{equ_partial_correlation_def}. A  simple alternative approach would be to consider the samples 
of each block in \eqref{equ_proc_model} as i.i.d. samples from a marginal CIG $\cig^{(\blkidx)}$ and apply 
existing GMS methods for i.i.d.\ samples to obtain estimates for the marginal CIGs. We can then obtain an 
estimate for the global CIG $\cig$ by using the union of the edge sets in each marginal CIG $\cig^{(\blkidx)}$. 
More precisely, a ``naive'' estimate $\widehat{\mathcal{N}}^{(\rm naive)}(i)$ for the neighborhood $\mathcal{N}(i)$ of some node 
$i \in \nodes$ can be obtained from the union of the block-wise neighborhood estimates 
$\widehat{\mathcal{N}}^{(\blkidx)}(i)$, for $\blkidx\!=\!1,\ldots,\nrblocks$.

In Figure \ref{fig:clime}, we compare the error rate achieved by our pooled approach to this naive approach. 
In particular, we use a constrained $\ell_1$ minimization approach (referred to as ``CLIME'') to estimate the support of the 
sparse precision matrix \cite{clime} within each block. From Figure \ref{fig:clime} we obtain that the pooled 
approach \eqref{equ_neigborhood_est} clearly outperforms the naive approach. This result should not come 
as a surprise since the pooled estimator \eqref{equ_neigborhood_est} allows to cope with few blocks with very 
small connection strength (of connected nodes in the CIG) as long as the average connection strength 
(see \eqref{equ_partial_correlation_def}) is large enough. In contrast, the naive approach is likely to fail if there 
is at least one block of samples which does not allow accurate GMS. 

\begin{figure}[htbp]
\begin{center}
\begin{tikzpicture}
 \tikzset{x=5cm,y=3cm,every path/.style={>=latex},node style/.style={circle,draw}}
      \node [anchor=south] at (0,1.1) {$P_{\rm err}$};
           \foreach \label/\labelval in {0/$0$,0.25/$0.25$,0.5/$0.5$,0.75/$0.75$,1/$1$}
        { 
          \draw (1pt,\label) -- (-1pt,\label) node[left] {\labelval};
        }       
             \foreach \label/\labelval in {0/$0$,0.25/$0.25$,0.5/$0.5$,0.75/$0.75$,1/$1$} 
        { 
          \draw (\label,1pt) -- (\label,-2pt) node[below] {\labelval};
        }
          \draw[->] (0,0) -- (1.2,0);
      \node [] at (0.5,-0.3) {$\samplesize'/100$};
      \draw[->] (0,0) -- (0,1.1);
      
    \csvreader[ head to column names,%
                late after head=\xdef\aold{\x}\xdef\bold{\y},%
                after line=\xdef\aold{\x}\xdef\bold{\y}]%
                {ScaledSample_p64_task00.csv}{}
                {\draw [line width=0.2mm] (\aold/100, \bold) -- (\x/100,\y) node {\large $+$};
    }
        \csvreader[ head to column names,%
                late after head=\xdef\aold{\x}\xdef\bold{\y},%
                after line=\xdef\aold{\x}\xdef\bold{\y}]%
                {ScaledSample_p128_task00.csv}{}
                {\draw [line width=0.2mm] (\aold/100, \bold) -- (\x/100,\y) node {\large $\times$};
    }
        \csvreader[ head to column names,%
                late after head=\xdef\aold{\x}\xdef\bold{\y},%
                after line=\xdef\aold{\x}\xdef\bold{\y}]%
                {ScaledSample_p256_task00.csv}{}
                {\draw [line width=0.2mm] (\aold/100, \bold) -- (\x/100,\y) node {\large $\circ$};
    }
    
            \csvreader[ head to column names,%
                late after head=\xdef\aold{\x}\xdef\bold{\y},%
                after line=\xdef\aold{\x}\xdef\bold{\y}]%
                {ScaledSample_p512_task00.csv}{}
                {\draw [line width=0.2mm] (\aold/100, \bold) -- (\x/100,\y) node {\large $\star$};
    }
    
            \csvreader[ head to column names,%
                late after head=\xdef\aold{\x}\xdef\bold{\y},%
                after line=\xdef\aold{\x}\xdef\bold{\y}]%
                {CLIMEScaledSample64.csv}{}
                {\draw [line width=0.2mm,dotted] (\aold/100, \bold) -- (\x/100,\y) node {\large $+$};
    }
                \csvreader[ head to column names,%
                late after head=\xdef\aold{\x}\xdef\bold{\y},%
                after line=\xdef\aold{\x}\xdef\bold{\y}]%
                {CLIMEScaledSample128.csv}{}
                {\draw [line width=0.2mm,dotted] (\aold/100, \bold) -- (\x/100,\y) node {\large $\times$};
    }
                \csvreader[ head to column names,%
                late after head=\xdef\aold{\x}\xdef\bold{\y},%
                after line=\xdef\aold{\x}\xdef\bold{\y}]%
                {CLIMEScaledSample256.csv}{}
                {\draw [line width=0.2mm,dotted] (\aold/100, \bold) -- (\x/100,\y) node {\large $\circ$};
    }
                    \csvreader[ head to column names,%
                late after head=\xdef\aold{\x}\xdef\bold{\y},%
                after line=\xdef\aold{\x}\xdef\bold{\y}]%
                {CLIMEScaledSample512.csv}{}
                {\draw [line width=0.2mm,dotted] (\aold/100, \bold) -- (\x/100,\y) node {\large $\star$};
    }
     \vspace*{-3mm}
\end{tikzpicture}
  \end{center}
  \caption{Error rate \eqref{equ_def_error_rate} achieved by estimating the neighborhood $\mathcal{N}(2)$ using the 
  pooled estimator \eqref{equ_neigborhood_est} (solid curves) and by the union of the neighbourhoods obtained by 
  applying CLIME \cite{clime} to each block in \eqref{equ_proc_model} 
  separately and then forming a union (over all blocks $\blkidx\!=\!1,\ldots,\nrblocks$) of all block-wise neighbourhood 
  estimates (dotted curves), }
    \label{fig:clime}
\end{figure}

As pointed out in Section \ref{sec_GMS_neighbrohoud_regr}, the sparse regression method \eqref{equ_neigborhood_est} 
becomes intractable except for very small number $\coeffdim$ of process components in \eqref{equ_proc_model} and 
sparsity level $\sparsity$ of the underlying CIG (see Assumption \ref{aspt_cig_sparse}). A computationally more tractable 
GMS method can be obtained by replacing (relaxing) the non-convex penalty term $\lambda | \setT|$ in \eqref{equ_neigborhood_est} 
by a convex approximation. The group Lasso is obtained by a particular choice for this convex approximation as \cite{BachConsistency2008} 
\begin{equation}
\label{equ_group_Lasso}
\hat{\vw}\!=\!\argmin_{\substack{\vw^{(\blkidx)}\in\mathbb{R}^{\coeffdim} \\ \hspace*{-4mm}w^{(\blkidx)}_{i}=0}} \sum_{\blkidx=1}^{\nrblocks} \big\| \vx^{\blknot}_{i}\!-\! 
\sum_{j =1}^{\coeffdim} w^{(\blkidx)}_{j} \vx^{(\blkidx)}_{j} \big\|_{2}^{2}+ \lambda \sum_{j=1}^{\coeffdim} \| \vw_{j} \|_{2}.
\end{equation} 
\begin{figure}[htbp]
\begin{minipage}[t]{0.7\columnwidth}
\begin{tikzpicture}
 \tikzset{x=5cm,y=3cm,every path/.style={>=latex},node style/.style={circle,draw}}
      \node [anchor=south] at (0,1.1) {$P_{\rm err}$};
           \foreach \label/\labelval in {0/$0$,0.25/$0.25$,0.5/$0.5$,0.75/$0.75$,1/$1$}
        { 
          \draw (1pt,\label) -- (-1pt,\label) node[left] {\labelval};
        }       
             \foreach \label/\labelval in {0/$0$,0.25/$0.25$,0.5/$0.5$,0.75/$0.75$,1/$1$} 
        { 
          \draw (\label,1pt) -- (\label,-2pt) node[below] {\labelval};
        }
          \draw[->] (0,0) -- (1.2,0);
      \node [] at (0.5,-0.3) {$\samplesize/500$};
      \draw[->] (0,0) -- (0,1.1);
      
    \csvreader[ head to column names,%
                late after head=\xdef\aold{\x}\xdef\bold{\y},%
                after line=\xdef\aold{\x}\xdef\bold{\y}]%
                {SampleGLasso_p64_task00_runs100.csv}{}  
                {\draw [line width=0.2mm] (\aold/500, \bold) -- (\x/500,\y) node {\large $+$};
    }
        \csvreader[ head to column names,%
                late after head=\xdef\aold{\x}\xdef\bold{\y},%
                after line=\xdef\aold{\x}\xdef\bold{\y}]%
                {SampleGLasso_p128_task00_runs100.csv}{}  
                {\draw [line width=0.2mm] (\aold/400, \bold) -- (\x/400,\y) node {\large $\times$};
    }
        \csvreader[ head to column names,%
                late after head=\xdef\aold{\x}\xdef\bold{\y},%
                after line=\xdef\aold{\x}\xdef\bold{\y}]%
                {SampleGLasso_p256_task00_runs100.csv}{}  
                {\draw [line width=0.2mm] (\aold/400, \bold) -- (\x/400,\y) node {\large $\circ$};
    }
        \csvreader[ head to column names,%
                late after head=\xdef\aold{\x}\xdef\bold{\y},%
                after line=\xdef\aold{\x}\xdef\bold{\y}]%
                {SampleGLasso_p512_task00_runs100.csv}{}  
                {\draw [line width=0.2mm] (\aold/400, \bold) -- (\x/400,\y) node {\large $\star$};
    }
    \node [] at (0.5,-0.7) {(a)};
        \vspace*{-3mm}
\end{tikzpicture}
\end{minipage}

\begin{minipage}[t]{0.7\columnwidth}
\begin{tikzpicture}
 \tikzset{x=5cm,y=3cm,every path/.style={>=latex},node style/.style={circle,draw}}
       \node [anchor=south] at (0,1.1) {$P_{\rm err}$};
           \foreach \label/\labelval in {0/$0$,0.25/$0.25$,0.5/$0.5$,0.75/$0.75$,1/$1$}
        { 
          \draw (1pt,\label) -- (-1pt,\label) node[left] {\labelval};
        }       
             \foreach \label/\labelval in {0/$0$,0.25/$0.25$,0.5/$0.5$,0.75/$0.75$,1/$1$} 
        { 
          \draw (\label,1pt) -- (\label,-2pt) node[below] {\labelval};
        }
          \draw[->] (0,0) -- (1.2,0);
      \node [] at (0.5,-0.3) {$\samplesize'/150$};
         \node [] at (0.5,-0.7) {(b)};
      \draw[->] (0,0) -- (0,1.1);
      
    \csvreader[ head to column names,%
                late after head=\xdef\aold{\x}\xdef\bold{\y},%
                after line=\xdef\aold{\x}\xdef\bold{\y}]%
                {ScaledSampleGLasso_p64_task00_runs100.csv}{}
                {\draw [line width=0.2mm] (\aold/150, \bold) -- (\x/150,\y) node {\large $+$};
    }
        \csvreader[ head to column names,%
                late after head=\xdef\aold{\x}\xdef\bold{\y},%
                after line=\xdef\aold{\x}\xdef\bold{\y}]%
                {ScaledSampleGLasso_p128_task00_runs100.csv}{}
                {\draw [line width=0.2mm] (\aold/150, \bold) -- (\x/150,\y) node {\large $\times$};
    }
        \csvreader[ head to column names,%
                late after head=\xdef\aold{\x}\xdef\bold{\y},%
                after line=\xdef\aold{\x}\xdef\bold{\y}]%
                {ScaledSampleGLasso_p256_task00_runs100.csv}{}
                {\draw [line width=0.2mm] (\aold/150, \bold) -- (\x/150,\y) node {\large $\circ$};
    }
            \csvreader[ head to column names,%
                late after head=\xdef\aold{\x}\xdef\bold{\y},%
                after line=\xdef\aold{\x}\xdef\bold{\y}]%
                {ScaledSampleGLasso_p512_task00_runs100.csv}{}
                {\draw [line width=0.2mm] (\aold/150, \bold) -- (\x/150,\y) node {\large $\star$};
    }
         
    \vspace*{-3mm}
\end{tikzpicture}
\end{minipage}
        \vspace*{-3mm}
  \caption{Empirical error rate $P_{\rm err} = \widehat{{\rm P}}\big\{ \widehat{\mathcal{N}}^{(\rm gLasso)}(2) \neq \mathcal{N}(2) \big\}$ 
  (see \eqref{equ_def_error_rate}) incurred by the neighborhood estimate \eqref{equ_def_gLasso_est} applied to a process with 
  chain structured CIG of size $\coeffdim=64$ (``$+$''), $\coeffdim\!=\!128$ (``$\times$''), $\coeffdim\!=\!256$ (``$\circ$'') 
  and $\coeffdim\!=\!512$ (``$\star$''). (a) Error rate as a function of the sample size $\samplesize$. (b) Error rate as 
  a function of scaled sample size $\samplesize'=\samplesize / \log \procdim$. Error rate has been obtained using $K=100$ 
  simulation runs.} \label{fig:GLasso}
\end{figure}

To obtain an estimate for the neighbourhood $\mathcal{N}(i)$ from the estimator \eqref{equ_group_Lasso}, we threshold 
the squared block norms $\| \hat{\vw}_{j} \|_{2}^{2}\!=\!\sum_{\blkidx=1}^{\nrblocks} \big( \hat{w}_{j}^{\blknot} \big)^{2}$ at the level 
$\eta\!=\!\rho_{\rm min}^{2}/2$,
\begin{equation}
\label{equ_def_gLasso_est}
\widehat{\mathcal{N}}^{(\rm gLasso)}(i) \defeq \{ j \in \{1,\ldots,\coeffdim\} \setminus \{i\}: \| \hat{\vw}_{j} \|_{2}^{2} \geq \eta \}. 
\end{equation} 
In Figure \ref{fig:GLasso}, we depict the error rate incurred by \eqref{equ_def_gLasso_est} for a process with chain-structured CIG (see Figure \ref{chainGraph}). 
While Figure \ref{fig:GLasso}-(a) shows the error rate as a function of the original sample size $\samplesize$, Figure \ref{fig:GLasso}-(b) 
displays the error rate as a function of the scaled sampled size $\samplesize'=\samplesize \rho_{\rm min}^{2}/ \log \procdim$. In agreement 
with our theoretical findings (see Theorem \ref{thm_upper_bound_err_prob}), the error rate of the estimator \eqref{equ_def_gLasso_est} 
seems to be mainly determined by the scaled sample size $\samplesize'$ as indicated in Figure \ref{fig:GLasso}-(b). 

A comparison of Figure \ref{fig:GLasso}-(b) with Figure \ref{fig:edgeCase}-(b) reveals that the estimator \eqref{equ_def_gLasso_est} 
requires more process samples than the estimator \eqref{equ_neigborhood_est} to ensure a prescribed error rate. However, the 
estimator \eqref{equ_def_gLasso_est} can be implemented by applying computationally efficient convex optimization methods to 
solve the group Lasso \eqref{equ_group_Lasso} (see \cite{DistrOptStatistLearningADMM}). 

\newpage
\subsection{Pedestrian Counts}
\label{ped_counts}

In this experiment we applied the sparse regression estimator \eqref{equ_neigborhood_est} to hourly pedestrian counts 
collected in the city of Turku (Finland). The city operates pedestrian counting devices at certain locations in the city 
center (see Figure \ref{fig:counter}). 
\begin{figure}[htbp]
\begin{center}
\begin{minipage}[t]{0.4\textwidth}
 \includegraphics[width=7cm]{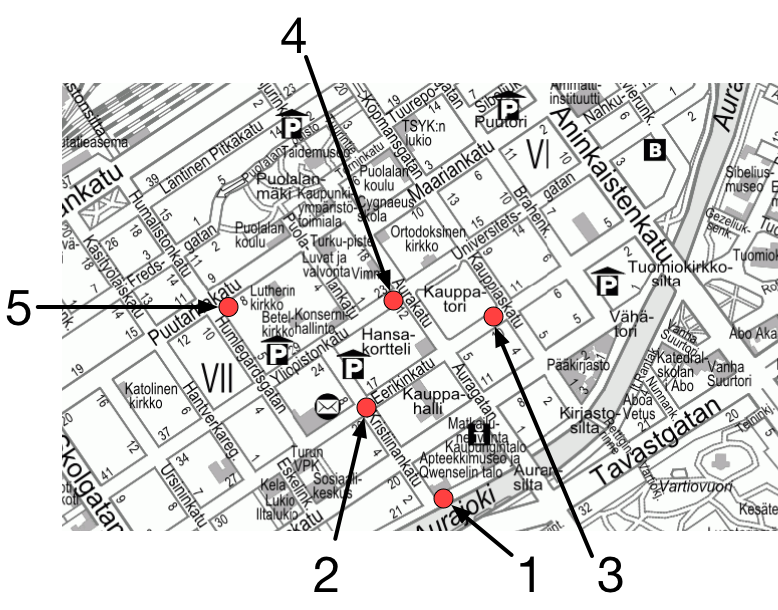}  
 \end{minipage}
\begin{minipage}[t]{0.4\textwidth}
 \includegraphics[width=7cm]{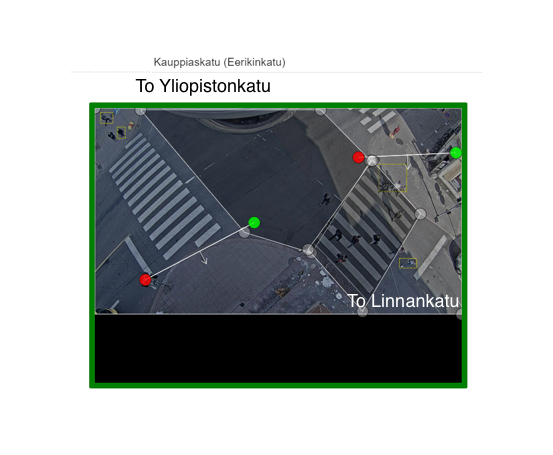}  
 \end{minipage}
\caption{Left: Map of Turku city including the locations of pedestrian count devices (depicted as red dots). Right: Snapshot generated by  
counting device ``$3$'' in order to count the number of pedestrians crossing each of two (virtual) counting lines in a particular direction.}
\label{fig:counter}
\end{center}
\end{figure}
The counting devices (based on cameras) measure the number of pedestrians which pass one of two counting lines in 
a certain direction (see Figure \ref{fig:counter}). 

We have been provided with hourly count data obtained from $\coeffdim=5$ different counting devices located in 
the city center of Turku (see Figure \ref{fig:counter}) and collected since $23$rd of July, $2018$. For each counting 
device, we compute the average count $z^{(i)}[\timeidx]$, for $i\!=\!1,\ldots,\coeffdim$ at time $\timeidx$. We depict the 
average count $z^{(1)}[\timeidx]$ in Figure \ref{fig:pooled2}, which indicates a seasonal component with period $24$. 
This is not too surprising as we expect the pedestrian movements for different days to be similar for the same daytime. 
\begin{figure}[htbp]
\begin{center}
\begin{tikzpicture}
 \tikzset{x=6cm,y=3cm,every path/.style={>=latex},node style/.style={circle,draw}}
    \node [anchor=south] at (0,1.1) {$z^{(1)}[\timeidx]/1000$};
    \foreach \label/\labelval in {0/$0$,0.25/$0.25$,0.5/$0.5$,0.75/$0.75$,1/$1$}
    {\draw (1pt,\label) -- (-1pt,\label) node[left] {\labelval};}       
     \foreach \label/\labelval in {0/$0$,0.25/$0.25$,0.5/$0.5$,0.75/$0.75$,1/$1$} 
     {\draw (\label,1pt) -- (\label,-2pt) node[below] {\labelval};}
     \draw[->] (0,0) -- (1.1,0);
     \node [] at (0.5,-.25) {$\timeidx/240 {\rm h}$};
     \draw[->] (0,0) -- (0,1.1);
    \csvreader[ head to column names,%
                late after head=\xdef\aold{\x}\xdef\bold{\y},%
                after line=\xdef\aold{\x}\xdef\bold{\y}]%
                {Raw_Counts_1.csv}{}
                {\draw [line width=0.5mm] (\aold/240, \bold/1000) -- (\x/240,\y/1000) node {\large $\circ$};}         
         \vspace*{-7mm}
\end{tikzpicture}
  \end{center}
        \vspace*{-7mm}
  \caption{Hourly pedestrian counts (averaged over two counting lines) at location $j=1$ as indicated in Figure \ref{fig:counter}.}
    \label{fig:pooled2}
\end{figure}

In order to remove the seasonal component we difference the time series $z^{(j)}[\timeidx]$
at lag $24$ to obtain the time series (see \cite[Chapter 1.4]{Brockwell91}) 
\begin{equation} 
\label{equ_def_difference}
\tilde{z}^{(i)}[\timeidx] \defeq z^{(i)}[\timeidx\!+\!24]\!-\!z^{(i)}[\timeidx] \mbox{ for } i=1,\ldots,\coeffdim.
\end{equation} 
We depict the time series $\tilde{z}^{(i)}[\timeidx]$ in Figure \ref{fig:diff}, which suggests that is is reasonable to 
model $\tilde{z}^{(i)}[\timeidx]$ as a stationary time series (or discrete time process). 

As discussed in Section \ref{SecProblemFormulation}, we can transform a stationary process into a process 
conforming to our non-stationary model \eqref{equ_proc_model} by applying a DFT. We compute the DFT of 
the difference time series $\tilde{z}^{(i)}[\timeidx]$ (see \eqref{equ_def_difference}) using a period $\samplesize\!=\!3072$ 
to obtain the vector-valued samples 
\begin{align} 
\label{equ_dft_samples}
\vx[\timeidx]&\!=\!\big( x^{(1)}[\timeidx],\ldots, x^{(\coeffdim)}[\timeidx] \big)^{T} \mbox{, with } \nonumber \\ 
x^{(i)}[\timeidx]& \!\defeq\!\sum_{\timeidx'=1}^{\samplesize} \tilde{z}^{(i)}[\timeidx'] \exp \big( - 2 \pi \iota (\timeidx'\!-\!1)(\timeidx\!-\!1)/\samplesize \big) 
\end{align} 
for $\timeidx=1,\ldots,\samplesize$ and $\iota \defeq \sqrt{-1}$. 

\begin{figure}[htbp]
\begin{center}
\hspace*{-5mm}
\begin{tikzpicture}
 \tikzset{x=6cm,y=3cm,every path/.style={>=latex},node style/.style={circle,draw}}
    \node [anchor=south] at (0,0.6) {$\tilde{z}^{(1)}[\timeidx]/1000$};
    \foreach \label/\labelval in  {0/$0$,0.25/$0.25$,0.5/$0.5$,-0.25/$-0.25$,-0.5/$-0.5$} 
    {\draw (1pt,\label) -- (-1pt,\label) node[left] {\labelval};}       
    
     \draw[->] (0,0) -- (1.1,0);
     \node [] at (1.1,-0.1) {$\timeidx/240 {\rm h}$};
     \draw[->] (0,-0.5) -- (0,0.6);
      
    \csvreader[ head to column names,%
                late after head=\xdef\aold{\x}\xdef\bold{\y},%
                after line=\xdef\aold{\x}\xdef\bold{\y}]%
                {Diff_Counts_1.csv}{}
                {\draw [line width=0.1mm] (\aold/240, \bold/1000) -- (\x/240,\y/1000) node {\large $\circ$};}       
         \vspace*{-3mm}
\end{tikzpicture}
  \end{center}
  \caption{Differenced (at lag $24$) hourly pedestrian counts $\tilde{z}^{(1)}[\timeidx]$ (see \eqref{equ_def_difference}) 
  at location $j\!=\!1$ (see Figure \ref{fig:counter}).}
    \label{fig:diff}
\end{figure}

We model the samples $\vx[\timeidx]$ using \eqref{equ_proc_model} with a block-length $\blocklen\!=\!12$ which has been chosen 
based on the empirical autocorrelation functions of the differenced time series $\tilde{z}^{(i)}$ (see \eqref{equ_def_difference}).
In order to infer the neighbourhoods $\mathcal{N}(i)$ in the CIG underlying the count measurements, we compute the 
test statistic 
\begin{align} 
\label{equ_def_stat_ped}
Z(\setT) & \!\defeq\!(1/\samplesize)\sum\limits_{\blkidx=1}^{\nrblocks} \|\mP_{\setT^{\perp}}^\blknot \vx_i^\blknot \|_2^2  \nonumber \\
& = (1/\samplesize)\sum\limits_{\blkidx=1}^{\nrblocks}  \min_{w^{\blknot}_{j} \in \mathbb{R}} \big\| \vx_{i}^{\blknot} - \sum_{j \in \setT} w^{\blknot}_{j} \vx_{j}^{\blknot} \big\|^{2}_{2}, 
\end{align} 
with the DFT samples \eqref{equ_dft_samples} and varying candidate sets $\setT \subseteq \{1,\ldots,\coeffdim \} \setminus \{i \}$.
\footnote{While our analysis applies only to real-valued vector samples $\vx[\timeidx]$ in \eqref{equ_proc_model}, the vector samples 
\eqref{equ_dft_samples} obtained from a DFT are typically complex-valued. However, we expect our analysis to also apply to 
complex-valued samples in \eqref{equ_dft_samples} by applying straightforward modifications of our methods. In particular, 
we believe that the fundamental dependencies (see \eqref{equ_suff_cond_ell_3}) between required sample size 
$\samplesize$ on number $\coeffdim$ of process components, sparsity $\sparsity$ and average connection strength 
$\rho^{2}_{\rm min}$ to remain valid when allowing the samples in \eqref{equ_proc_model} to be 
complex-valued Gaussian vectors.} 

Since we neither know the maximum node degree (sparsity) $\sparsity$, nor a lower bound $\rho^{2}_{\rm min}$ on the 
average connection strength, we cannot directly implement the sparse regression estimator \eqref{equ_neigborhood_est}. 
Instead, we try to estimate the neighborhood $\mathcal{N}(i)$ of node $i \in \nodes$ by evaluating the decay of the 
score $\mathcal{E}(\sparsity) \defeq \min_{|\setT|\!=\!\sparsity} Z(\setT)$ using the statistic \eqref{equ_def_stat_ped} 
(which is the first component in the objective function of the sparse regression estimator \eqref{equ_neigborhood_est}).

\begin{figure}[htbp]
\begin{center}
\scalebox{1}{
\begin{tikzpicture}
 \tikzset{x=2cm,y=3cm,every path/.style={>=latex},node style/.style={circle,draw}}
    \node [anchor=south] at (0,1.1) {$\mathcal{E}(s)/\mathcal{E}(0)$};
    \foreach \label/\labelval in  {0/$0$,0.25/$0.25$,0.5/$0.5$,0.75/$0.75$,1/$1$} 
    {\draw (1pt,\label) -- (-1pt,\label) node[left] {\labelval};}       
     \foreach \label/\labelval in  {0/$0$,1/$1$,2/$2$,3/$3$}
     {\draw (\label,1pt) -- (\label,-2pt) node[below] {\labelval};}
     \draw[->] (0,0) -- (3.1,0);
     \node [anchor=north] at (1.2,-0.2) {$s$};
     \draw[->] (0,0) -- (0,1.1);
      
    \csvreader[ head to column names,%
                late after head=\xdef\aold{\x}\xdef\bold{\y},%
                after line=\xdef\aold{\x}\xdef\bold{\y}]%
                {min_val_i_1.csv}{}
                {\draw [line width=0.1mm] (\aold-1, \bold) -- (\x-1,\y) node {\large $\circ$};}       
         \vspace*{-6mm}
\end{tikzpicture}}
  \end{center}
       \vspace*{-6mm}
  \caption{Score $\mathcal{E}(\sparsity)\!=\!\min_{|\setT|\!=\!\sparsity} Z(\setT)$ achieved by minimizing the statistic \eqref{equ_def_stat_ped} 
  (for node $i\!=\!1$) over all candidate sets $\setT$ with a prescribed size $\sparsity$.}
    \label{fig:score}
\end{figure}

In Figure \ref{fig:score}, we depict the score $\mathcal{E}(\sparsity)$ obtained for node $i\!=\!2$. We then 
choose the neighborhood size $\sparsity$ as the smallest number such that 
$\mathcal{E}(\sparsity)\!-\!\mathcal{E}(\sparsity\!+\!1)\!<\!2 (\mathcal{E}(\sparsity)\!-\!\widetilde{\mathcal{E}}(\sparsity))$ 
with the ``auxiliary score''
\begin{equation} 
\widetilde{\mathcal{E}}(\sparsity)\!=\!(1/\samplesize)\sum\limits_{\blkidx=1}^{\nrblocks}  \min_{w^{\blknot}_{j}} \big\| \vx_{i}^{\blknot}\!-\! 
\big( \sum_{j \in \setT'} w^{\blknot}_{j} \vx_{j}^{\blknot}\!+\!c \mathbf{f}^{\blknot} \big) \big\|^{2}_{2}.
\end{equation} 
Here, the index set $\setT'$ is chosen as $\setT'\!=\!\argmin_{|\setT|\!=\!\sparsity} Z(\setT)$. 

The idea behind comparing $\mathcal{E}(\sparsity)$ with $\widetilde{\mathcal{E}}(s)$ is to test if adding another process 
component to the components in $\mathcal{T}'$ yields a reduction in the statistic $\mathcal{E}(\sparsity\!+\!1)$ which is 
at least twice as large as the reduction of $\mathcal{E}(\sparsity)$ achieved by adding a ``fake'' pedestrian count signal 
obtained by i.i.d.\ uniformly distributed random variables $f[\timeidx] \sim \mathcal{U}[0,U]$. The interval size $U$ 
is chosen in order to match the empirical variance of the pedestrian counts $z^{(i)}[\timeidx]$. 

We have obtained the following estimates for the neighbourhoods in the CIG underlying the pedestrian count data: 
\begin{align}
\widehat{\mathcal{N}}(1)& \!=\!\{2\}, \widehat{\mathcal{N}}(2)\!=\!\{3,5\}, \widehat{\mathcal{N}}(3)\!=\!\{2,4\} ,  \nonumber \\ 
\widehat{\mathcal{N}}(4)& \!=\!\{3,5\} , \widehat{\mathcal{N}}(5)\!=\!\{2,4\}.  \nonumber
\end{align} 
In Figure \ref{fig:pedest_est_cig}, we depicted the CIG estimate obtained by placing an edge between nodes $i,j \in \nodes$ 
if either $j \in \widehat{\mathcal{N}}(i)$ or $i \in \widehat{\mathcal{N}}(j)$. The estimated graph structure seems 
well-aligned with the local road network.

\begin{figure}[htbp]
\begin{center}
 \includegraphics[width=7cm]{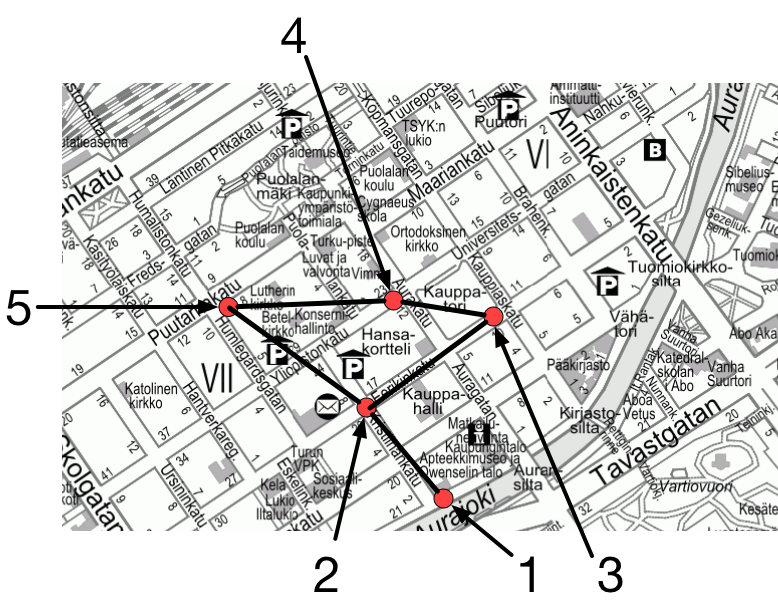}  
\caption{Map of Turku city including the locations of pedestrian count devices (depicted as red dots). The links between the 
count devices indicate the presence of an edge in the estimated CIG underlying the count data.}
\label{fig:pedest_est_cig}
\end{center}
\end{figure}

\section*{Acknowledgement}
We thank Tuomas Piippo and Arvi Leino from the city of Turku for help with the pedestrian 
count data. 

\bibliographystyle{IEEEtran}
\bibliography{/Users/junga1/Literature.bib}

\begin{thebibliography}{10}
\providecommand{\url}[1]{#1}
\csname url@samestyle\endcsname
\providecommand{\newblock}{\relax}
\providecommand{\bibinfo}[2]{#2}
\providecommand{\BIBentrySTDinterwordspacing}{\spaceskip=0pt\relax}
\providecommand{\BIBentryALTinterwordstretchfactor}{4}
\providecommand{\BIBentryALTinterwordspacing}{\spaceskip=\fontdimen2\font plus
\BIBentryALTinterwordstretchfactor\fontdimen3\font minus
  \fontdimen4\font\relax}
\providecommand{\BIBforeignlanguage}[2]{{%
\expandafter\ifx\csname l@#1\endcsname\relax
\typeout{** WARNING: IEEEtran.bst: No hyphenation pattern has been}%
\typeout{** loaded for the language `#1'. Using the pattern for}%
\typeout{** the default language instead.}%
\else
\language=\csname l@#1\endcsname
\fi
#2}}
\providecommand{\BIBdecl}{\relax}
\BIBdecl

\bibitem{SandrMoura2014b}
A.~Sandryhaila and J.~M.~F. Moura, ``Big data analysis with signal processing
  on graphs: Representation and processing of massive data sets with irregular
  structure,'' \emph{IEEE Signal Processing Magazine}, vol.~31, no.~5, pp.
  80--90, Sept 2014.

\bibitem{NetworkLasso}
D.~Hallac, J.~Leskovec, and S.~Boyd, ``Network lasso: Clustering and
  optimization in large graphs,'' in \emph{Proc. SIGKDD}, 2015, pp. 387--396.

\bibitem{koller2009probabilistic}
D.~Koller, N., and Friedman, \emph{Probabilistic Graphical Models: Principles
  and Techniques}, ser. Adaptive computation and machine learning.\hskip 1em
  plus 0.5em minus 0.4em\relax MIT Press, 2009.

\bibitem{BigDataNetworksBook}
S.~Cui, A.~Hero, Z.-Q. Luo, and J.~Moura, Eds., \emph{Big Data over
  Networks}.\hskip 1em plus 0.5em minus 0.4em\relax Cambridge Univ. Press,
  2016.

\bibitem{WikiData2014}
D.~Vrande\v{c}i\'{c} and M.~Kr\"{o}tzsch, ``Wikidata: A free collaborative
  knowledgebase,'' \emph{Commun. ACM}, vol.~57, no.~10, pp. 78--85, Sep. 2014.

\bibitem{Sadeghi2017}
A.~Sadeghi, C.~Lange, M.~Vidal, and S.~Auer, ``Communication metadata using
  knowledge graphs,'' in \emph{Lecture Notes in Computer Science}.\hskip 1em
  plus 0.5em minus 0.4em\relax Springer, 2017.

\bibitem{DistrOptStatistLearningADMM}
S.~Boyd, N.~Parikh, E.~Chu, B.~Peleato, and J.~Eckstein, \emph{{D}istributed
  {O}ptimization and {S}tatistical {L}earning via the {A}lternating {D}irection
  {M}ethod of {M}ultipliers}.\hskip 1em plus 0.5em minus 0.4em\relax Hanover,
  MA: Now {P}ublishers, 2010, vol.~3, no.~1.

\bibitem{Ambos2018}
H.~Ambos, N.~Tran, and A.~Jung, ``Classifying big data over networks via the
  logistic network lasso,'' in \emph{Proc. 52nd Asilomar Conference on Signals,
  Systems, and Computers}.\hskip 1em plus 0.5em minus 0.4em\relax
  10.1109/ACSSC.2018.8645260, 2018.

\bibitem{GraphModExpFamVarInfWainJor}
M.~J. Wainwright and M.~I. Jordan, \emph{Graphical {M}odels, {E}xponential
  {F}amilies, and {V}ariational {I}nference}, ser. Foundations and Trends in
  Machine Learning.\hskip 1em plus 0.5em minus 0.4em\relax Hanover, MA: Now
  {P}ublishers, 2008, vol.~1, no. 1--2.

\bibitem{LauritzenGM}
S.~L. Lauritzen, \emph{Graphical Models}.\hskip 1em plus 0.5em minus
  0.4em\relax Oxford, UK: Clarendon Press, 1996.

\bibitem{RavWainLaff2010}
P.~Ravikumar, M.~J. Wainwright, and J.~Lafferty, ``High-dimensional {I}sing
  model selection using $\ell_{1}$-regularized logistic regression,''
  \emph{Ann. Stat.}, vol.~38, no.~3, pp. 1287--1319, 2010.

\bibitem{FriedHastieTibsh2008}
J.~H. Friedmann, T.~Hastie, and R.~Tibshirani, ``Sparse inverse covariance
  estimation with the graphical lasso,'' \emph{Biostatistics}, vol.~9, no.~3,
  pp. 432--441, Jul. 2008.

\bibitem{JMLRHub}
K.~M. Tan, P.~London, K.~Mohan, S.-I. Lee, M.~Fazel, and D.~Witten, ``Learning
  graphical models with hubs,'' \emph{Jour. Mach. Learning Res.}, vol.~15,
  no.~10, pp. 3297--3331, Oct. 2014.

\bibitem{BachJordan04}
F.~R. Bach and M.~I. Jordan, ``Learning graphical models for stationary time
  series,'' \emph{IEEE Trans. Signal Processing}, vol.~52, no.~8, pp.
  2189--2199, Aug. 2004.

\bibitem{CSGraphSelJournal}
A.~Jung, ``Learning the conditional independence structure of stationary time
  series: A multitask learning approach,'' \emph{IEEE Trans. Signal
  Processing}, vol.~63, no.~21, Nov. 2015.

\bibitem{HannakJung2014conf}
G.~Hannak, A.~Jung, and N.~G{\"o}rtz, ``On the information-theoretic limits of
  graphical model selection for {G}aussian time series,'' in \emph{Proc.
  EUSIPCO 2014}, Lisbon, Portugal, 2014.

\bibitem{JuHeck2014}
A.~Jung, R.~Heckel, H.~B\"{o}lcskei, and F.~Hlawatsch, ``Compressive
  nonparametric graphical model selection for time series,'' in \emph{Proc.
  IEEE ICASSP-2014}, Florence, Italy, May 2014.

\bibitem{JungGaphLassoSPL}
A.~{Jung}, G.~{Hannak}, and N.~{G{\"o}rtz}, ``{Graphical LASSO Based Model
  Selection for Time Series},'' \emph{IEEE Sig. Proc. Letters}, vol.~22,
  no.~10, pp. 1781--1785, Oct. 2015.

\bibitem{Yang2015Nips}
E.~Yang and A.~Lozano, ``Robust gaussian graphical modeling with the trimmed
  graphical lasso,'' in \emph{Advances in Neural Information Processing Systems
  28}, 2015, pp. 2602--2610.

\bibitem{WangWain2010}
W.~Wang, M.~J. Wainwright, and K.~Ramchandran, ``Information-theoretic bounds
  on model selection for {G}aussian {M}arkov random fields,'' in \emph{Proc.
  IEEE ISIT-2010}, Austin, TX, Jun. 2010, pp. 1373--1377.

\bibitem{RavWainRaskYu2011}
P.~Ravikumar, M.~J. Wainwright, and B.~Raskutti, G.~Yu, ``High-dimensional
  covariance estimation by minimizing $\ell_{1}$-penalized log-determinant
  divergence,'' \emph{Electronic Journal of Statistics}, vol.~5, pp. 935--980,
  2011.

\bibitem{Kipnis2018}
A.~Kipnis, A.~Goldsmith, and Y.~Eldar, ``The distortion rate function of
  cyclostationary gaussian processes,'' \emph{IEEE Trans. Inform. Theory},
  vol.~64, no.~5, pp. 3810--3824, 2018.

\bibitem{Mallat98}
S.~Mallat, G.~Papanicolaou, and Z.~Zhang, ``Adaptive covariance estimation of
  locally stationary processes,'' \emph{Ann. Statist.}, vol.~26, no.~1, pp.
  1--47, 1998.

\bibitem{TimeFrequencyAnalysisBoashash}
B.~Boashash, Ed., \emph{Time {F}requency {S}ignal {A}nalysis and {P}rocessing:
  A {C}omprehensive {R}eference}.\hskip 1em plus 0.5em minus 0.4em\relax
  Amsterdam, The Netherlands: Elsevier, 2003.

\bibitem{Dahlhaus2000}
R.~Dahlhaus, ``Graphical interaction models for multivariate time series,''
  \emph{Metrika}, vol.~51, pp. 151--172, 2000.

\bibitem{Eichler03}
M.~Eichler, R.~Dahlhaus, and J.~Sandk{\"u}hler, ``Partial correlation analysis
  for the identification of synaptic connections,'' \emph{Biol Cybern.},
  vol.~89, no.~4, 2003.

\bibitem{DanaherGroupGLASSO2014}
P.~Danaher, P.~Wang, and D.~M. Witten, ``The joint graphical lasso for inverse
  covariance estimation across multiple classes,'' \emph{J. R. Stat. Soc. B},
  vol.~76, pp. 373--397, 2014.

\bibitem{LeeLiu2015}
W.~Lee and Y.~Liu, ``Joint estimation of multiple precision matrices with
  common structures,'' \emph{Journal of Machine Learning Research}, vol.~16,
  no.~1, pp. 1035--1062, 2015.

\bibitem{Peterson2013}
C.~Peterson, F.~Stingo, and M.~Vannucci, ``Bayesian inference of multiple
  gaussian graphical models,'' \emph{Journal of the American Statistical
  Association}, vol. 110, Apr. 2015.

\bibitem{MeinBuhl2006}
N.~Meinshausen and P.~B{\"u}hlmann, ``High-dimensional graphs and variable
  selection with the {L}asso,'' \emph{Ann. Stat.}, vol.~34, no.~3, pp.
  1436--1462, 2006.

\bibitem{DavidsonLevin2005}
E.~Davidson and M.~Levin, ``Gene regulatory networks,'' \emph{Proc. Natl. Acad.
  Sci.}, vol. 102, no.~14, Apr. 2005.

\bibitem{Brockwell91}
P.~J. Brockwell and R.~A. Davis, \emph{Time Series: Theory and Methods}.\hskip
  1em plus 0.5em minus 0.4em\relax Springer New York, 1991.

\bibitem{Starica2005}
C.~Starica and C.~Granger, ``Nonstationarities in stock returns,'' \emph{The
  Review of Economics and Statistics}, vol.~87, no.~3, pp. 495--502, 2005.

\bibitem{Wahlberg2007}
P.~Wahlberg and M.~Hansson, ``Kernels and multiple windows for estimation of
  the wigner-ville spectrum of gaussian locally stationary processes,''
  \emph{IEEE Transactions on Signal Processing}, vol.~55, no.~10, 2007.

\bibitem{Dahlhaus98}
R.~Dahlhaus and L.~Giraitis, ``On the optimal segment length for parameter
  estimates for locally stationary time series,'' \emph{Journal of Time Series
  Analysis}, vol.~19, no.~6, 1998.

\bibitem{Dahlhaus2009}
R.~Dahlhaus, ``Local inference for locally stationary time series based on the
  empirical spectral measure,'' \emph{Journal of Econometrics}, 2009.

\bibitem{jung-specesttit}
A.~Jung, G.~Taub\"ock, and F.~Hlawatsch, ``Compressive spectral estimation for
  nonstationary random processes,'' \emph{IEEE Trans. Inf. Theory}, vol.~59,
  no.~5, pp. 3117--3138, May 2013.

\bibitem{Velez90}
E.~F. Velez and R.~G. Absher, ``Spectral estimation based on the wigner-ville
  representation,'' \emph{Signal Processing}, vol.~20, 1990.

\bibitem{Wain2009TIT}
M.~J. Wainwright, ``Information-theoretic limits on sparsity recovery in the
  high-dimensional and noisy setting,'' \emph{IEEE Trans. Inf. Theory},
  vol.~55, no.~12, pp. 5728--5741, Dec. 2009.

\bibitem{Gallager13}
R.~G. Gallager, \emph{Stochastic Processes: Theory for Applications}.\hskip 1em
  plus 0.5em minus 0.4em\relax Cambridge University Press, 2013.

\bibitem{papoulis}
A.~Papoulis and S.~U. Pillai, \emph{Probability, Random Variables, and
  Stochastic Processes}, 4th~ed.\hskip 1em plus 0.5em minus 0.4em\relax New
  York: Mc-Graw Hill, 2002.

\bibitem{Wain2009TITSharpThresholds}
M.~J. Wainwright, ``Sharp thresholds for high-dimensional and noisy sparsity
  recovery using $\ell_{1}$-constrained quadratic programming ({L}asso),''
  \emph{IEEE Trans. Inf. Theory}, vol.~55, no.~5, pp. 2183--2202, May 2009.

\bibitem{Loh2017}
P.-L. Loh and M.~J. Wainwright, ``Support recovery without incoherence: A case
  for nonconvex regularization,'' \emph{Ann. Statist.}, vol.~45, no.~6, pp.
  2455--2482, 2017.

\bibitem{clime}
T.~Cai, W.~Liu, and X.~Luo, ``A constrained $\ell_{1}$ minimization approach to
  sparse precision matrix estimation,'' \emph{Journal of the American
  Statistical Association}, vol. 106, no. 494, pp. 594--607, 2011.

\bibitem{BachConsistency2008}
F.~R. Bach, ``Consistency of the group lasso and multiple kernel learning,''
  \emph{J. Mach. Lear. Research}, vol.~9, pp. 1179--1225, 2008.

\bibitem{RauhutFoucartCS}
S.~Foucart and H.~Rauhut, \emph{A Mathematical Introduction to Compressive
  Sensing}.\hskip 1em plus 0.5em minus 0.4em\relax New York: Springer, 2012.

\end{thebibliography}

\onecolumn
\vspace*{-0mm}
\section{Proof of the Main Result} 
\vspace*{-0mm}
\label{sec_proof_main_result}

We now verify Theorem \ref{thm_upper_bound_err_prob} by analyzing the probability $\prob\{ \error \}$ of the error event 
$\error$ (see \eqref{equ_def_error_no_superset}) when \eqref{equ_neigborhood_est} fails to deliver the correct neighbourhood 
$\mathcal{N}(i)$ of a particular node $i \in \nodes$ of the CIG $\cig$.  
Let us introduce the shorthands 
\begin{align}
\label{equ_def_projection_test_statistic}
\hspace*{-3mm}\errevent \!\defeq\!\{ Z(\mathcal{N}(i))\! + \lambda s_{i} >\!Z(\mathcal{T}) + \lambda |\setT|  \} \mbox{, with } 
Z(\setT)\!\defeq\!\frac{1}{\samplesize} \sum\limits_{\blkidx=1}^{\nrblocks} \|\mP_{\setT^{\perp}}^\blknot \vx_i^\blknot \|_2^2.
\end{align} 

It will be convenient to denote the set of all subsets of $\{1,\ldots,\coefflen\}$ of size at most $\sparsity$ 
but different from the true neighbourhood $\mathcal{N}(i)$ by
\begin{align*}
\setofT \defeq \{ \setT \subseteq \{1,\ldots,\coeffdim \} : |\setT| \leq \sparsity , \setT \neq \mathcal{N}(i)\} .
\end{align*}
Moreover, for given $\ell_{1}, t \leq \sparsity$, denote
\begin{equation} 
\label{equ_def_N_ell_1_t}
\mathcal{N}(\ell_{1},t)  \defeq \{ \setT \in \setofT:  |\setT|=t, |\mathcal{N}(i) \setminus \setT|\!=\!\ell_{1} \}. 
\end{equation} 
Thus, the set $\mathcal{N}(\ell_{1},t) \subseteq \setofT$ collects all the index sets in $\setofT$ with a prescribed size 
$t=|\setT|$ and overlap $\ell_{1} =  |\mathcal{N}(i) \setminus \setT|$ with the true neighbourhood $\mathcal{N}(i)$. 

An elementary combinatorial argument (see \cite[Sec. IV]{Wain2009TIT}) reveals that the number of these index sets is 
\begin{equation}
\label{equ_size_N_ell1_t}
N(\ell_{1},t) \defeq |\mathcal{N}(\ell_{1},t)|= \binom{s_{i}}{\ell_{1}} \binom{\coeffdim-s_{i}}{\ell_{2}}. 
\end{equation} 
with 
\begin{equation} 
\label{equ_def_ell_2}
\ell_{2} \defeq \ell_{1}+(t-s_{i}). 
\end{equation} 
Given a particular node $i \in \nodes$ with neighbourhood $\mathcal{N}(i)$, the quantities $\ell_{1}$ and $\ell_{2}$ are fully determined 
by the index set $\setT$. For notational convenience we will not make this dependence on $\setT$ explicit, i.e., we write $\ell_{1}$ and $\ell_{2}$ 
instead of $\ell_{1}(\setT)$ and $\ell_{2}(\setT)$. 
Note that 
\begin{equation}
\label{equ_ell_2_overlap_T}
\ell_{2} = | \setT \setminus \mathcal{N}(i) | \mbox{ and } \ell_{1}+\ell_{2} > 0 \mbox{ for every index set } \setT \in \mathcal{N}(\ell_{1},t).  
\end{equation}
Using the index set 
\begin{equation} 
\label{equ_size_I}
\mathcal{I} \defeq \{ (\ell_{1},t) \in \mathbb{Z}_{+}^{2}: \ell_{1} \leq s_{i}, t \leq \sparsity \} \setminus \{(0,s_{i}) \} \mbox{ with cardinality } |\mathcal{I}| \leq \sparsity^2,
\end{equation}
we can write
\begin{align}
\label{equ_set_of_t_subseteq}
\setofT & \subseteq   \bigcup_{(\ell_{1},t) \in \mathcal{I}} \mathcal{N}(\ell_{1},t).
\end{align}

Since the error event $\error$ (see \eqref{equ_def_error_no_superset}) can only occur if at least one 
of the events $\errevent$, for some $\setT \in \setofT$, occurs, 
\begin{equation}
\label{equ_subset_error_errevent}
\error \subseteq \bigcup_{\setT \in \setofT} \errevent,
\end{equation} 
implying, in turn via a union bound, 
\begin{equation} 
\label{equ_bound_union_err}
\prob \{ \error \} \stackrel{\eqref{equ_subset_error_errevent}}{\leq} \sum_{\setT \in \setofT} \prob \{ \errevent \} 
\stackrel{\eqref{equ_set_of_t_subseteq}}{\leq} \sum_{(\ell_{1},t) \in \mathcal{I}} \sum_{\setT \in \mathcal{N}(\ell_{1},t)} \prob \{ \errevent \}.
\end{equation} 

We now derive an upper bound $M(\ell_{1},t)$ on the individual probabilities $\prob\{ \errevent \}$ such that
\begin{equation}
\label{equ_bound_M_N_ell_1}
\prob\{ \errevent \} \leq M(\ell_{1},t) \mbox{ for any } \setT \in \mathcal{N}(\ell_{1},t). 
\end{equation} 
As the notation already indicates, the upper bound $M(\ell_{1},t)$ depends on the index set $\setT$ 
only via the overlap $\ell_{1}\!=\!|\mathcal{N}(i) \setminus \setT|$ and the size $t = | \setT |$. 

Combining \eqref{equ_set_of_t_subseteq}
with \eqref{equ_bound_union_err} implies, via a union bound,  
\begin{align}
\label{equ_prob_bound_M_N}
\log  \prob \{ \error \} & \stackrel{\eqref{equ_bound_union_err}}{\leq} 
 \log \sum_{ (\ell_{1},t) \in \mathcal{I} } \sum_{ \setT \in \mathcal{N}(\ell_{1},t) }  \prob \{ \errevent \} \nonumber \\[2mm]
  & \stackrel{\eqref{equ_bound_M_N_ell_1}}{\leq}\log \sum_{ (\ell_{1},t) \in \mathcal{I} } \sum_{ \setT \in \mathcal{N}(\ell_{1},t) } M(\ell_{1},t) \nonumber \\[2mm]
   & \leq  \log |\mathcal{I}| + \max_{(\ell_{1},t)\in \mathcal{I}} \hspace*{-0mm}  \big[ \log N(\ell_{1},t) \!+\! \log M(\ell_{1},t) \big] \nonumber \\[2mm]
    & \stackrel{\eqref{equ_size_N_ell1_t},\eqref{equ_size_I}}{\leq} 2 \log \sparsity + \max_{(\ell_{1},t)\in \mathcal{I}} 
    \hspace*{-0mm} \big[ \ell_{1} \log s_{i} + \ell_{2} \log (\coeffdim-s_{i}) \!+\! \log M(\ell_{1},t) \big] \nonumber \\[2mm]
     &  \leq 2 \log \sparsity + \max_{(\ell_{1},t)\in \mathcal{I}} \hspace*{-0mm} \big[ (\ell_{1} + \ell_{2}) \log \coeffdim \!+\! \log M(\ell_{1},t) \big]. 
\end{align} 

Our next goal is to find a sufficiently tight upper bound $M(\ell_{1},t)$ on the probabilities of the events  
$\prob\{ \errevent \}$ (see \eqref{equ_def_projection_test_statistic}) with some index set $\setT \in \mathcal{N}(\ell_{1},t)$. 
To this end, we make \eqref{equ_comp_2} more handy by stacking the (block-wise) noise 
vectors $\bfep_{i}^\blknot \in \mathbb{R}^{\blockLength}$ into the single noise vector
\begin{align}
\bfep_{i} \!=\! \big( \big(\bfep_{i}^{(1)}\big)^{T}, \ldots, \big(\bfep_{i}^{(\nrblocks)} \big)^{T} \big)^T \!\sim\! \mathcal{N}(\mathbf{0},\mC_{\bfep_{i}})  
 \mbox{ , with } \mC_{\bfep_{i}} \!=\!  {\rm blkdiag} \{(1/K^{(\blkidx)}_{i,i} ) \mathbf{I}_{\blocklen} \}_{\blkidx=1}^{\nrblocks}.\label{equ_def_single_noise}
\end{align}
By introducing the projection matrix  
\begin{equation}
\label{equ_def_mP_blk_set_T_perp}
\mP_{\setT^{\perp}} \defeq {\rm blkdiag} \{  \mathbf{P}^{\blknot}_{\setT^{\perp}} \}_{\blkidx=1}^{\nrblocks},  
\end{equation}
we can characterize the error event $\errevent$ in \eqref{equ_def_projection_test_statistic}, for any $\setT \in \mathcal{N}(\ell_{1},t)$ (see \eqref{equ_def_N_ell_1_t}), as 
\vspace*{0mm}
\begin{align}
\hspace{-0.5mm}\errevent & \!=\! \big\{ Z(\mathcal{N}(i)) \!-\! (1/\samplesize) \norm[2]{\mP_{\setT^{\perp}} \bfep_{i}}^2 
\!>\! Z(\setT)\!-\!(1/\samplesize) \norm[2]{\mP_{\setT^{\perp}} \bfep_{i}}^2  + \lambda (t-s_{i})\big\}.
\label{eq_proj_erro_event_separated}
\end{align}

In order to derive the upper bound $M(\ell_{1},t)$ let us, for some number $\delta > 0$ whose precise value to be 
chosen in what follows, define the two error events
\begin{subequations}
\begin{align}
\label{eq_proj_erro_event_1}
\hspace*{-3mm}\mathcal{E}_1(\delta) \!&\defeq\! \big\{  Z(\mathcal{N}(i)) - (1/\samplesize) \norm[2]{\mP_{\setT^{\perp}} \bfep_{i}}^2 \geq \delta\!+\!(\lambda/2)(t\!-\!s_{i}) \big\},\\[3mm]
\hspace*{-3mm}\mathcal{E}_2(\delta) \!&\defeq\! \big\{ Z(\setT)\!-\!(1/\samplesize) \norm[2]{\mP_{\setT^{\perp}} \bfep_{i}}^2 +(\lambda/2)(t\!-\!s_{i}) \leq 2\delta \big\}.
\label{eq_proj_erro_event_2}
\end{align}
\end{subequations}
By \eqref{eq_proj_erro_event_separated}, an error $\errevent$ can only occur if either $\mathcal{E}_1(\delta)$ 
or $\mathcal{E}_2(\delta)$ occurs, i.e., $\errevent \subseteq \mathcal{E}_1(\delta) \cup \mathcal{E}_2(\delta)$. 
Therefore, by a union bound, 
\begin{align}
\label{equ_union_bound_E_1_2}
\prob \{ \errevent \} & \leq \prob \{ \mathcal{E}_1(\delta) \}\!+\!\prob \{ \mathcal{E}_2(\delta) \} \nonumber \\[3mm]
& =  \expect \big \{ \prob \{ \mathcal{E}_1(\delta) | \vx_{\setT} \} \big\} \!+\! \expect \big\{ \prob \{ \mathcal{E}_2(\delta) | \vx_{\setT} \} \big\}, 
\end{align}
where we condition on the components $\vx_{\setT} = \{ \vx_{i} \}_{i \in \setT}$ (cf. \eqref{equ_component_i}). 

We will now bound each of the two summands in \eqref{equ_union_bound_E_1_2} separately. 
To this end, we will use the singular value decomposition (SVD)
\begin{equation}
\mP_{\setT^{\perp}} \mathbf{C}^{1/2}_{\tilde{x}_{i}} = \mathbf{U} {\rm diag}\{d_{j} \}_{j=1}^{\samplesize} \mathbf{V}^{T} 
\label{equ_svd}
\end{equation} 
with the singular values $d_{j} \in \mathbb{R}_{+}$ and the singular vectors in the columns of the orthonormal matrices 
$\mathbf{U} \in \mathbb{R}^{\samplesize \times \samplesize}$ and 
$\mathbf{V} \in \mathbb{R}^{\samplesize \times \samplesize}$ (i.e., $\mathbf{U}\mathbf{U}^{T} = \mathbf{V} \mathbf{V}^{T} = \mathbf{I}$). 
The singular values $d_{j}$, which satisfy
\begin{equation}
\label{equ_upper_bound_sing_values_beta}
d_{j} \stackrel{\eqref{equ_distribution_tilde_vx},\eqref{equ_bound_simga_2_blokdx}}{\leq} \sqrt{\beta}
\end{equation} 
will play a prominent role in controlling the probabilities of the error events $\mathcal{E}_1(\delta)$ and 
$\mathcal{E}_2(\delta)$ (see \eqref{eq_proj_erro_event_1}, \eqref{eq_proj_erro_event_2}). In particular, 
we will analyze the probabilities of those events for the choice $\delta = m_{3}/4$ with 
\begin{align}
m_{3} & \defeq  \expect \{ (1/\samplesize) \| \mP_{\setT^{\perp}}\tilde{\vx}_{i} \|^{2}_{2} \mid \vx_{\setT} \} \nonumber \\[2mm]
& \stackrel{(a)}{=} (1/\samplesize) \trace \{\mathbf{C}^{1/2}_{\tilde{x}_{i}} \mP_{\setT^{\perp}} \mathbf{C}^{1/2}_{\tilde{x}_{i}} \} \nonumber \\[2mm]
& \stackrel{\eqref{equ_svd}}{=} (1/\samplesize) \sum_{j=1}^{\samplesize} d_{j}^{2},  
\label{equ_events_3_m3}
\end{align}
where in step $(a)$ we used the statistical independence of $\tilde{\vx}_{i}$ and $\vx_{\setT}$ (cf.\ \eqref{equ_comp_1}).  

The quantity $m_{3}$ measures the minimum achievable error when approximating the process component $\vx_{i}$ 
(see \eqref{equ_component_i}) using a linear combination of the process components $\vx_{\setT}\!=\!\{ \vx_{j} \}_{j \in \setT}$. 
A lower bound on $m_{3}$ can be obtained via the minimum average connection strength $\rho^{2}_{\rm min}$ (see Assumption \ref{aspt_minimum_par_cor}). 
Indeed,
\begin{align}
m_{3} &\stackrel{\eqref{equ_events_3_m3}}{=} (1/\samplesize) \trace \{ \mC^{1/2}_{\tilde{\vx}_{i}} \mP_{\setT^{\perp}}\mC^{1/2}_{\tilde{\vx}_{i}}\} \nonumber \\[3mm] 
& = (1/\samplesize) \trace \{ \mC_{\tilde{\vx}_{i}} \mP_{\setT^{\perp}}\} \nonumber \\[3mm] 
& \stackrel{\eqref{equ_def_vx_i_long}}{=} (1/\samplesize) \sum_{\blkidx=1}^{\nrblocks}  \trace \big\{ \mP^{(\blkidx)}_{\setT^{\perp}} \tilde{\sigma}_{b}^{2} \mathbf{I} \big\}  \nonumber \\[3mm]
& \stackrel{\eqref{equ_def_P_b_setT}}{=} (1/\samplesize) \sum_{\blkidx=1}^{\nrblocks} \tilde{\sigma}^{2}_{\blkidx} (\blocklen\!-\!|\setT|). \label{equ_derviation_m_3_1111}
\end{align} 
This can be further developed by using the lower bound \eqref{equ_events_3_m3_2} for the variance $\tilde{\sigma}^{2}_{\blkidx}$, 
\begin{align} 
m_{3} &\stackrel{\eqref{equ_derviation_m_3_1111},\eqref{equ_events_3_m3_2}}{\geq} \sum_{j\in \mathcal{N}(i)\setminus \setT} 
\sum_{\blkidx=1}^{\nrblocks}  (K_{i,j}^{(\blkidx)}/K_{i,i}^{(\blkidx)})^2 (\blocklen\!-\!|\setT|)   /\samplesize \nonumber \\
&\stackrel{\eqref{equ_partial_correlation_def},\eqref{equ_rho_min_aspt}}{\geq} \ell_{1} \nrblocks \rho^{2}_{\rm min} (\blocklen\!-\!|\setT|)   /\samplesize  \nonumber \\
& \stackrel{\eqref{equ_sparsity_aspt}}{\geq} (2/3) \ell_{1} \rho^{2}_{\rm min}.
\label{equ_bound_m_3_111}
\end{align} 
For the choice $\delta = m_{3}/4$ this implies, in turn, 
\begin{equation}
\label{equ_lower_boud_dela_rho_min}
\delta \geq (1/6) \ell_{1} \rho^{2}_{\rm min}. 
\end{equation} 
In order to upper bound the probability of the event $\mathcal{E}_1(\delta)$, observe 
\begin{align}
Z(\mathcal{N}(i)) & \stackrel{\eqref{equ_def_projection_test_statistic}}{=} (1/\samplesize) \sum\limits_{\blkidx=1}^{\nrblocks} 
\norm[2]{\mP_{\mathcal{N}(i)^{\perp}}^{\blknot} \vx_i^{\blknot}}^2 \nonumber\\
&\stackrel{\eqref{equ_comp_2}}{=} (1/\samplesize) \sum\limits_{\blkidx=1}^{\nrblocks} \bigg \|\mP_{\mathcal{N}(i)^{\perp}}^{\blknot} 
\big( \sum_{i \in \mathcal{N}(i)} a_{j} \vx_j^{\blknot}  + \bfep_{i}^{\blknot} \big)\bigg \|_2^2 \nonumber\\
&\stackrel{\eqref{equ_def_mP_blk_set_T_perp},\eqref{equ_def_single_noise}}{=} (1/\samplesize) \norm[2]{\mP_{\mathcal{N}(i)^{\perp}} \bfep_{i}}^2.
\label{equ_proj_neighbor_set}
\end{align}
Hence,
\begin{align}
\label{equ_bound_prob_E_1}
\prob \{ \mathcal{E}_1(\delta) \mid \vx_{\setT} \} &\stackrel{\eqref{eq_proj_erro_event_1}}{=} 
\prob \big\{  Z(\mathcal{N}(i))\!-\!(1/\samplesize) \norm[2]{\mP_{\setT^{\perp}} \bfep_{i}}^2\!\geq\! \delta\!+\!(\lambda/2) (t\!-\!s_{i}) \!\mid\! \vx_{\setT}\big\} \nonumber\\[2mm]
& \hspace*{-5mm} \stackrel{\eqref{equ_proj_neighbor_set}}{=}  \prob \big\{ (1/\samplesize) \norm[2]{\mP_{\mathcal{N}(i)^{\perp}} \bfep_{i}}^2 \!-\! 
 (1/\samplesize)  \norm[2]{\mP_{\setT^{\perp}} \bfep_{i}}^2  \geq \delta\!+\!(\lambda/2) (t\!-\!s_{i}) \mid \vx_{\setT}  \big\} \nonumber\\[2mm]
& \hspace*{-5mm} \stackrel{\lambda=\rho^{2}_{\rm min}/6}{=}    \prob \big\{ (1/\samplesize) \norm[2]{\mP_{\mathcal{N}(i)^{\perp}} \bfep_{i}}^2 \!-\! 
 (1/\samplesize)  \norm[2]{\mP_{\setT^{\perp}} \bfep_{i}}^2  \geq \delta\!+\!(\rho^{2}_{\rm min}/12) (t\!-\!s_{i}) \mid \vx_{\setT}  \big\} 
 \nonumber\\[2mm]
& \hspace*{-5mm} \stackrel{\eqref{equ_lower_boud_dela_rho_min},\eqref{equ_def_ell_2}}{\leq}  \prob \big\{ (1/\samplesize) \norm[2]{\mP_{\mathcal{N}(i)^{\perp}} \bfep_{i}}^2 \!-\! 
 (1/\samplesize)  \norm[2]{\mP_{\setT^{\perp}} \bfep_{i}}^2  \geq \rho^{2}_{\rm min} (\ell_{1}+\ell_{2})/12\mid \vx_{\setT}  \big\}.%
\end{align}
By elementary properties of projections in Euclidean spaces \cite[Appx. A]{Wain2009TIT}
\begin{align} 
\label{equ_elem_identieis_proj}
 \norm[2]{\mP_{\mathcal{N}(i)^{\perp}} \bfep_{i}}^2 \!-\!  \norm[2]{\mP_{\setT^{\perp}} \bfep_{i}}^2 & =
 \norm[2]{\mP_{\setT} \bfep_{i}}^2  -  \norm[2]{\mP_{\mathcal{N}(i)} \bfep_{i}}^2  \nonumber \\ 
 & = \norm[2]{\big( \mP_{\setT} - \mP_{\setT \cap \mathcal{N}(i)} \big)\bfep_{i}}^2  -  \norm[2]{\big( \mP_{\mathcal{N}(i)} - \mP_{\setT \cap \mathcal{N}(i)} \big)\bfep_{i}}^2, 
\end{align} 
with 
\begin{equation}
\nonumber
 \mP_{\setT} \defeq  {\rm blkdiag} \big\{ \mP_{\setT}^{\blknot} \big\}_{\blkidx=1}^{\nrblocks}.
\end{equation} 
Combining \eqref{equ_elem_identieis_proj} with \eqref{equ_bound_prob_E_1}, 
\begin{align}
\label{equ_bound_prob_E_1112}
\prob \{ \mathcal{E}_1(\delta)\mid \vx_{\setT}  \} \leq  \prob \big\{ (1/\samplesize)  \norm[2]{\big( \mP_{\setT}
 - \mP_{\setT \cap \mathcal{N}(i)} \big)\bfep_{i}}^2  \geq \rho^{2}_{\rm min} (\ell_{1}+\ell_{2})/12\mid \vx_{\setT}  \big\}. 
\end{align}
with 
\begin{equation} 
\big( \mP_{\setT} - \mP_{\setT \cap \mathcal{N}(i)} \big) =  {\rm blkdiag} \big\{ \mP_{\setT}^{\blknot}
 -\mP^{\blknot}_{\setT \cap \mathcal{N}(i)}   \big\}_{\blkidx=1}^{\nrblocks}. \nonumber
\end{equation}  
The matrix $\mP^{\blknot} _{\setT} - \mP^{\blknot} _{\setT \cap \mathcal{N}(i)} \in \mathbb{R}^{\blocklen \times \blocklen}$ 
is a random (since it depends on $\vx_{\setT}\!=\!\{ \vx_{j} \}_{j \in \setT}$) orthogonal projection matrix on a subspace of 
dimension at most $\ell_{2} = | \setT \setminus \mathcal{N}(i)|$ (cf. \eqref{equ_ell_2_overlap_T}), i.e., 
\begin{equation} 
\label{equ_diff_projection_matrix}
\mP^{\blknot} _{\setT} - \mP^{\blknot} _{\setT \cap \mathcal{N}(i)}  = \sum_{j=1}^{\ell_{2}} \tilde{a}^{\blknot}_{j} \vu^{\blknot}_{j} \big( \vu^{\blknot}_{j} \big)^{T}
\end{equation} 
with some coefficients $\tilde{a}^{\blknot}_{j} \in \{0,1\}$ and orthonormal vectors $\{ \vu_{j}^{\blknot} \in \mathbb{R}^{\blocklen} \}_{j=1,\ldots,\ell_{2}}$. 
Inserting \eqref{equ_diff_projection_matrix} into \eqref{equ_bound_prob_E_1112}, 
\begin{align}
\label{equ_bound_prob_E_11123}
\prob \{ \mathcal{E}_1(\delta)\mid \vx_{\setT}  \} & \leq   \prob \big\{ (1/\samplesize)  \norm[2]{\big( \mP_{\setT} 
- \mP_{\setT \cap \mathcal{N}(i)} \big)\bfep_{i}}^2  
\geq \rho^{2}_{\rm min} (\ell_{1}+\ell_{2})/12\mid \vx_{\setT}  \big\} \nonumber \\[2mm]
& = \prob \big\{ (1/\samplesize) \sum_{\blkidx=1}^{\nrblocks}  \norm[2]{\big( \mP^{(\blkidx)}_{\setT} 
- \mP^{(\blkidx)}_{\setT \cap \mathcal{N}(i)} \big)\bfep^{(\blkidx)}_{i}}^2  
\geq \rho^{2}_{\rm min} (\ell_{1}+\ell_{2})/12\mid \vx_{\setT}  \big\}  \nonumber \\[2mm]
& \stackrel{\eqref{equ_diff_projection_matrix}}{=}  \prob \big\{ (1/\samplesize) \sum_{\blkidx=1}^{\nrblocks}  
\sum_{j=1}^{\ell_{2}} \tilde{a}^{\blknot}_{j}  \big( z^{\blknot}_{j} \big)^{2}  
\geq \rho^{2}_{\rm min} (\ell_{1}+\ell_{2})/12\mid \vx_{\setT}  \big\} 
\end{align}
with $z^{\blknot}_{j} = \big( \vu^{\blknot}_{j} \big)^{T} \bfep^{\blknot}_{i} \sim \mathcal{N}(0,1/K^{\blknot}_{i,i})$ (conditioned on $\vx_{\setT}$). 
Then, as can be verified easily, 
\begin{align}
\label{equ_sum_prof_E_T_1}
 (1/\samplesize) \sum_{\blkidx=1}^{\nrblocks}  \sum_{j=1}^{\ell_{2}} \tilde{a}^{\blknot}_{j}  \big( z^{\blknot}_{j} \big)^{2}  
 = \sum_{n=1}^{\samplesize} \tilde{a}_{n} z_{n}^{2} \mbox{, with } z_{n} \sim \mathcal{N}(0,1)
\end{align}
and coefficients $\tilde{a}_{n} \in [0, \beta/\samplesize]$ (cf. \eqref{equ_bounds_eigvals_asspt3}) satisfying 
\begin{equation}
\label{equ_sum_tilde_a_j_bound_24}
\sum_{n=1}^{\samplesize} \tilde{a}_{n} \stackrel{(a)}{\leq} \ell_{2} \beta \nrblocks /\samplesize \stackrel{\eqref{equ_condition_rho_min_kappa}}{\leq} \ell_{2} \rho^{2}_{\rm min}/24. 
\end{equation} 
Here, step $(a)$ can be verified by taking (conditional, w.r.t.\ $\vx_{\setT}\!=\!\{ \vx_{j} \}_{j \in \setT}$) expectations 
of \eqref{equ_sum_prof_E_T_1} and using $|a_{j}^{\blknot}| \leq 1$, $1/K^{\blknot}_{i,i} \stackrel{\eqref{equ_bounds_eigvals_asspt3}}{\leq} \beta$.  

Inserting \eqref{equ_sum_prof_E_T_1} into \eqref{equ_bound_prob_E_11123}, 
\begin{align}
\label{equ_bound_prob_E_111234}
 \prob \{ \mathcal{E}_1(\delta)\mid \vx_{\setT}  \} & \leq \prob \big\{ \sum_{j=1}^{\samplesize} \tilde{a}_{j} z_{j}^{2} 
  \geq \rho^{2}_{\rm min} (\ell_{1}+\ell_{2})/12\mid \vx_{\setT}  \big\}   \nonumber \\
 & =  \prob \big\{ \sum_{j=1}^{\samplesize} \tilde{a}_{j} z_{j}^{2} -  \sum_{j=1}^{\samplesize} \tilde{a}_{j}  
 \geq \rho^{2}_{\rm min} (\ell_{1}+\ell_{2})/12 - \sum_{j=1}^{\samplesize} \tilde{a}_{j}  \mid \vx_{\setT}  \big\}    \nonumber \\ 
 \nonumber \\
 & \stackrel{\eqref{equ_sum_tilde_a_j_bound_24}}{\leq} \prob \big\{ \sum_{j=1}^{\samplesize} \tilde{a}_{j} z_{j}^{2} -  \sum_{j=1}^{\samplesize} \tilde{a}_{j}  
 \geq \rho^{2}_{\rm min} (\ell_{1}+\ell_{2})/24  \mid \vx_{\setT}  \big\} \nonumber \\
 & \stackrel{z_{j}|\vx_{\setT} \sim \mathcal{N}(0,1)}{\leq} \prob \big\{ \sum_{j=1}^{\samplesize} \tilde{a}_{j} z_{j}^{2} -   \expect \{ \sum_{j=1}^{\samplesize} \tilde{a}_{j} z_{j}^{2} | \vx_{\setT} \}
 \geq \rho^{2}_{\rm min} (\ell_{1}+\ell_{2})/24  \mid \vx_{\setT}  \big\}.  
\end{align} 
We now apply Lemma \ref{lem_sub_exp_lin} to \eqref{equ_bound_prob_E_111234} using the choice
\begin{equation} 
\label{equ_def_eta_case_I}
\eta\!\defeq\!\rho^{2}_{\rm min} (\ell_{1}\!+\!\ell_{2})/24,
\end{equation} 
$a_{j} \defeq \tilde{a}_{j}$ and $b_{j} \defeq 0$ (cf. \eqref{equ_def_y_lem_sub_exp_lin}). 
This yields 
\begin{align}
\prob \{ \mathcal{E}_{1} (\delta) \mid \vx_{\setT}  \} & \stackrel{\eqref{equ_bound_prob_E_111234},\eqref{lem_1_1_lin}}{\leq} 
2 \exp \bigg(- \frac{ \eta^{2}/8}{\sum_{j=1}^{\samplesize} \tilde{a}^{2}_{j} \!+\! \eta \max\limits_{j=1,\ldots,\samplesize} \tilde{a}_{j}  }   \bigg) \nonumber \\ 
& \leq 2 \exp \bigg(- \frac{\samplesize \eta^{2}/(8\beta)}{\sum_{j=1}^{\samplesize} \tilde{a}_{j} \!+\! \eta }   \bigg), 
\label{equ_err_event_11233}
\end{align} 
where the second inequality uses $\max\limits_{j=1,\ldots,\samplesize} \tilde{a}_{j}  \stackrel{\eqref{equ_sum_prof_E_T_1}}{\leq} \beta/\samplesize$. 
Combining 
\begin{equation} 
\label{equ_sum_tilde_a_j}
\sum_{j=1}^{\samplesize} \tilde{a}_{j}  \stackrel{\eqref{equ_sum_tilde_a_j_bound_24}}{\leq} \ell_{2} \rho_{\rm min} /24 \stackrel{\eqref{equ_def_eta_case_I}}{\leq} \eta 
\end{equation} 
with \eqref{equ_err_event_11233}, we arrive at 
\begin{align} 
\prob \{ \mathcal{E}_{1} (\delta)   \}  & = \expect \big\{ \prob \big\{ \mathcal{E}_{1} (\delta)  \mid \vx_{\setT}  \big\} \big\}\nonumber \\
& \stackrel{\eqref{equ_err_event_11233},\eqref{equ_sum_tilde_a_j}}{\leq} 2 \exp \bigg(- \samplesize \eta/(16\beta)   \bigg)  \nonumber \\ 
& \stackrel{\eqref{equ_def_eta_case_I}}{=}  2 \exp \bigg(- \samplesize \rho_{\rm min} (\ell_{1}\!+\!\ell_{2}) /(24 \cdot 16\beta)   \bigg).   \label{equ_upper_bound_E_1_CaseI}   
\end{align} 


To upper bound the probability of $\mathcal{E}_{2}(\delta)$ (cf.\ \eqref{eq_proj_erro_event_2}), consider
\begin{align}
\label{equ_expre_E2}
\hspace*{-3mm}\prob \big\{ \mathcal{E}_2(\delta) \mid \vx_{\setT} \big \}  \!&\stackrel{\eqref{eq_proj_erro_event_2}}{\defeq}\!  
\prob \bigg \{ Z(\setT)\!-\!(1/\samplesize) \norm[2]{\mP_{\setT^{\perp}} \bfep_{i}}^2 + (\lambda/2) (t\!-\!s_{i}) \leq 2\delta  \mid \vx_{\setT} \bigg\} \nonumber \\[2mm]
& \hspace*{-2mm} \stackrel{ \eqref{equ_def_projection_test_statistic}}{=}  \prob \bigg\{ (1/\samplesize)\vx^{T}_i \mP_{\setT^{\perp}} \vx_{i}\!-\!(1/\samplesize) \bfep^{T}_{i}  
\mP_{\setT^{\perp}} \bfep_{i}+  (\lambda/2) (t\!-\!s_{i})  \!\leq\!2 \delta  \mid \vx_{\setT} \bigg\} \nonumber \\[2mm]
& \hspace*{-2mm} \stackrel{ \eqref{equ_comp_1}}{=} \prob \bigg\{ (1/\samplesize)    \tilde{\vx}^{T}_i \mP_{\setT^{\perp}} \tilde{\vx}_{i}\!+\!(2/\samplesize)
\tilde{\vx}^{T}_{i}  \mP_{\setT^{\perp}} \bfep_{i} +  (\lambda/2) (t\!-\!s_{i}) \!\leq\! 2 \delta  \mid \vx_{\setT} \bigg\} 
\end{align} 
with $\bfep_{i}=(\varepsilon_{1},\ldots,\varepsilon_{\samplesize})^{T}$ (cf.\ \eqref{equ_def_single_noise}) and $\tilde{\vx}_{i}$ (cf.\ \eqref{equ_def_vx_i_long}). 

By defining the random vector 
\begin{equation} 
\label{equ_def_vv}
\vv =(v_{1},\ldots,v_{\samplesize})^{T} \defeq \mathbf{V}^{T}   \mathbf{C}^{-1/2}_{\tilde{x}_{i}}\tilde{\vx}_{i},  \nonumber
\end{equation}
using the (random) orthonormal matrix $\mathbf{V} \in \mathbb{R}^{\samplesize \times \samplesize}$ constituted by the 
singular vectors of the matrix $\mP_{\setT^{\perp}} \mathbf{C}^{1/2}_{\tilde{x}_{i}}$  (cf.\ \eqref{equ_svd}), 
we can rewrite \eqref{equ_expre_E2} as
\begin{align}
 \label{equ_event_2_2345}
\prob \{ \mathcal{E}_2(\delta) \mid \vx_{\setT} \} &\!=\! \prob \bigg\{ \frac{1}{\samplesize} \sum_{j=1}^{\samplesize} v^{2}_{j} d^{2}_{j}\!+\!\frac{2}{\samplesize}
\sum_{j=1}^{\samplesize} v_{j} d_{j} \varepsilon_{j}\!\leq\!2 \delta -(\lambda/2) (t\!-\!s_{i}) \mid \vx_{\setT}  \bigg\} \nonumber \\ 
&  \hspace*{-10mm} \stackrel{\delta=m_{3}/4}{=}  \prob \bigg\{ \frac{1}{\samplesize} \sum_{j=1}^{\samplesize} v^{2}_{j} d^{2}_{j}\!+\!\frac{2}{\samplesize}
\sum_{j=1}^{\samplesize} v_{j} d_{j} \varepsilon_{j}- m_{3} \!\leq\! -m_{3}/2 -(\lambda/2) (t\!-\!s_{i})  \mid \vx_{\setT} \bigg\}.
\end{align} 
Note that, conditioned on $\vx_{\setT}$, the vector $\vv$ is standard Gaussian, i.e., $\vv \sim \mathcal{N}(\mathbf{0},\mathbf{I}_{\samplesize})$.
We now consider \eqref{equ_event_2_2345} for the particular choice $\lambda = \rho_{\rm min}/6$ which yields, using \eqref{equ_bound_m_3_111},
\begin{align}
\label{equ_event_2_23456}
\prob \big\{ \mathcal{E}_2(\delta)\mid \vx_{\setT} \big\} &\stackrel{\delta=m_{3}/4}{=}  \prob \bigg\{ \frac{1}{\samplesize} \sum_{j=1}^{\samplesize} v^{2}_{j} d^{2}_{j}\!+\!\frac{2}{\samplesize}
\sum_{j=1}^{\samplesize} v_{j} d_{j} \varepsilon_{j}- m_{3} \!\leq\! -m_{3}/2 -(\lambda/2) (t\!-\!s_{i})  \mid \vx_{\setT}  \bigg\} \nonumber \\
&\hspace*{-20mm} =  \prob \bigg\{ \frac{1}{\samplesize} \sum_{j=1}^{\samplesize} v^{2}_{j} d^{2}_{j}\!+\!\frac{2}{\samplesize}
\sum_{j=1}^{\samplesize} v_{j} d_{j} \varepsilon_{j}- m_{3} \!\leq\! -(3/8)m_{3} -((\lambda/2) (t\!-\!s_{i})+m_{3}/8)  \mid \vx_{\setT}  \bigg\} \nonumber \\
&\hspace*{-20mm} \stackrel{\eqref{equ_bound_m_3_111}}{\leq}  \prob \bigg\{ \frac{1}{\samplesize} \sum_{j=1}^{\samplesize} v^{2}_{j} d^{2}_{j}\!+\!\frac{2}{\samplesize}
\sum_{j=1}^{\samplesize} v_{j} d_{j} \varepsilon_{j}- m_{3} \!\leq\! -(3/8)m_{3} -((\lambda/2)  (t\!-\!s_{i})+\ell_{1} \rho_{\rm min}/12)  \mid \vx_{\setT}  \bigg\}\nonumber \\
&\hspace*{-20mm} \stackrel{\lambda\!=\!\rho_{\rm min}/6}{=}  \prob \bigg\{ \frac{1}{\samplesize} \sum_{j=1}^{\samplesize} v^{2}_{j} d^{2}_{j}\!+\!\frac{2}{\samplesize}
\sum_{j=1}^{\samplesize} v_{j} d_{j} \varepsilon_{j}- m_{3} \!\leq\! -(3/8)m_{3} -(\rho_{\rm min}/12) ( (t\!-\!s_{i})+\ell_{1})  \mid \vx_{\setT}  \bigg\}\nonumber \\
&\hspace*{-20mm} \stackrel{\eqref{equ_ell_2_overlap_T}}{=}  \prob \bigg\{ \frac{1}{\samplesize} \sum_{j=1}^{\samplesize} v^{2}_{j} d^{2}_{j}\!+\!\frac{2}{\samplesize}
\sum_{j=1}^{\samplesize} v_{j} d_{j} \varepsilon_{j}- m_{3} \!\leq\! -(3/8)m_{3} -(\rho_{\rm min}/12) \ell_{2}  \mid \vx_{\setT}  \bigg\}.
\end{align}

We will invoke Lemma \ref{lem_sub_exp_lin} to obtain an upper bound for $\prob \big\{ \mathcal{E}_2(\delta=m_{3}/4) \mid \vx_{\setT} \big\}$. 
To this end, in order to control the term $(2/\samplesize) \sum\limits_{j=1}^{\samplesize} v_{j} d_{j} \varepsilon_{j}$ in \eqref{equ_event_2_2345}, 
we condition on the event
\begin{equation} 
\label{equ_def_event_A}
\mathcal{A} \defeq \bigg\{ (1/\samplesize) \sum_{j=1}^{\samplesize} d_{j}^{2} \varepsilon^{2}_{j} 
\leq \underbrace {2 (\beta/\samplesize) \sum_{j=1}^{\samplesize} d_{j}^{2}}_{\stackrel{\eqref{equ_events_3_m3}}{=} 2 \beta  m_{3} } + \beta \ell_{2} (\rho_{\rm min}/12) \bigg\}
\end{equation} 
with the constant $\beta$ of Assumption \ref{aspt_eig_val}. The event $\mathcal{A}$ is, conditioned on $\vx_{\setT}$, statistically 
independent of $\tilde{\vx}_{i}$ (cf.\ \eqref{equ_comp_1}) since, loosely speaking, its definition \eqref{equ_def_event_A} 
involves only the random variables $\{ \varepsilon_{j} \}_{j=1,\ldots,\samplesize}$ which are statistically independent of $\tilde{\vx}_{i}$ (cf.\ \eqref{equ_comp_1}) 
and quantities (e.g., the singular values $d_{j}$) which are constant when conditioning on $\vx_{\setT}$ . 

We can upper bound the probability $\prob \big\{ \mathcal{E}_2(\delta=m_{3}/4) \mid \vx_{\setT} \big \}$ as 
\begin{align} 
\label{equ_bound_prob_E_2_total_prob}
\prob \{ \mathcal{E}_2(\delta) \mid \vx_{\setT} \} & = \prob \{ \mathcal{E}_2(\delta) | \mathcal{A}, \vx_{\setT} \} \prob \{ \mathcal{A} \mid \vx_{\setT}\}
+ \underbrace{\prob \{ \mathcal{E}_2(\delta) | \mathcal{A}^{c}, \vx_{\setT} \}}_{\leq 1}  \prob \{ \mathcal{A}^{c} \mid \vx_{\setT} \}  \nonumber \\[2mm]
& \leq  \prob \{ \mathcal{E}_2(\delta) | \mathcal{A}, \vx_{\setT} \} + \prob \{ \mathcal{A}^{c} \mid \vx_{\setT} \}. 
\end{align}
In order to control the probability $\prob \{ \mathcal{A}^{c} \mid \vx_{\setT} \}$ in \eqref{equ_bound_prob_E_2_total_prob}, we will invoke 
Lemma \ref{lem_sub_exp_lin}. To this end, observe  
\begin{align}
\prob \{ \mathcal{A}^{c} \mid \vx_{\setT} \} & = \prob \{ (1/\samplesize) \sum_{j=1}^{\samplesize} d_{j}^{2} \varepsilon^{2}_{j} 
\geq 2(\beta/\samplesize) \sum_{j=1}^{\samplesize} d_{j}^{2} + \beta \ell_{2} (\rho_{\rm min}/12)  \mid \vx_{\setT} \} \nonumber \\ 
& \hspace*{-20mm} \stackrel{(a)}{\leq}  \prob \{ (1/\samplesize) \sum_{j=1}^{\samplesize}  d_{j}^{2}  \varepsilon^{2}_{j}  
- \expect\{(1/\samplesize) \sum_{j=1}^{\samplesize} d_{j}^{2} \varepsilon^{2}_{j}  \mid \vx_{\setT}  \} 
\geq  (\beta/\samplesize) \sum_{j=1}^{\samplesize} d_{j}^{2} + \beta \ell_{2} (\rho_{\rm min}/12)  \mid \vx_{\setT}  \}, 
\label{equ_upper_bound_A_c}
\end{align} 
where $(a)$ is due to 
\begin{equation} 
\nonumber
\expect\bigg\{(1/\samplesize) \sum_{j=1}^{\samplesize} d_{j}^{2} \varepsilon^{2}_{j}  \mid \vx_{\setT}  \bigg\} = (1/\samplesize) 
\sum_{j=1}^{\samplesize} d_{j}^{2} \expect\{ \varepsilon^{2}_{j}  \mid \vx_{\setT}  \}   
\stackrel{\eqref{equ_bounds_eigvals_asspt3},\eqref{equ_def_single_noise}}{\leq} (\beta/\samplesize) \sum_{j=1}^{\samplesize} d_{j}^{2}. 
\end{equation} 
The random variables $\{\varepsilon_{j}\}_{j=1,\ldots,\samplesize}$ are, conditioned on $\vx_{\setT}$, i.i.d.\  
zero-mean Gaussian variables with variance $\sigma^{2}_{\varepsilon} \leq \beta$ (cf. \eqref{equ_def_single_noise}). Therefore, we can 
use the innovation representation 
\begin{equation}
\label{equ_innov_varepsilon}
\varepsilon_{j} = \tilde{b}_{j} z_{j}
\end{equation}
with i.i.d.\ standard Gaussian random variables $z_{j} \sim \mathcal{N}(0,1)$ and some coefficients $\tilde{b}_{j} \in [0,\sqrt{\beta}]$. 
Inserting \eqref{equ_innov_varepsilon} into \eqref{equ_upper_bound_A_c}, 
\begin{align}
\prob \{ \mathcal{A}^{c} \mid \vx_{\setT} \} &\leq \nonumber \\
& \hspace*{-20mm}  \prob \big\{ \sum_{j=1}^{\samplesize}  \tilde{b}^{2}_{j} d_{j}^{2} z^{2}_{j} \!-\!\expect\big\{\sum_{j=1}^{\samplesize}  \tilde{b}_{j} d_{j}^{2} z^{2}_{j} \mid \vx_{\setT}  \big\} 
\geq  \beta \big( \sum_{j=1}^{\samplesize} d_{j}^{2} +  \samplesize \ell_{2} (\rho_{\rm min}/12) \big) \mid \vx_{\setT}  \big\}. 
\label{equ_upper_bound_A_c_d}
\end{align} 
Applying \eqref{lem_1_1_lin}, using the choice $\eta \defeq  \beta \big( \sum_{j=1}^{\samplesize} d_{j}^{2} +  \samplesize \ell_{2} (\rho_{\rm min}/12) \big)$, 
$a_{j} \defeq \tilde{b}^{2}_{j} d_{j}^{2}$ and $b_{j} = 0$ (cf.\ \eqref{equ_def_y_lem_sub_exp_lin}) to \eqref{equ_upper_bound_A_c_d}, yields 
\begin{align}
\prob \{ \mathcal{A}^{c} \mid \vx_{\setT} \} & \stackrel{\eqref{lem_1_1_lin}}{\leq} \exp \bigg(-\frac{ \beta^{2}  \big( \sum_{j=1}^{\samplesize} d_{j}^{2} 
+  \samplesize \ell_{2} (\rho_{\rm min}/12) \big)^{2} / 8}{ \sum_{j=1}^{\samplesize} \tilde{b}_{j}^{4} d_{j}^{4}  + \beta \big( \sum_{j=1}^{\samplesize} d_{j}^{2}
+  \samplesize \ell_{2} (\rho_{\rm min}/12) \big)\max\limits_{j=1,\ldots,\samplesize} \tilde{b}_{j}^{2} d_{j}^{2}  } \bigg) \nonumber\\[2mm]
& \stackrel{\tilde{b}^{2}_{j}\leq \beta, \eqref{equ_upper_bound_sing_values_beta}}{\leq} \exp \bigg(-\frac{  \big( \sum_{j=1}^{\samplesize} d_{j}^{2}
+  \samplesize \ell_{2} (\rho_{\rm min}/12) \big)^{2} }{16 \beta \big( \sum_{j=1}^{\samplesize} d_{j}^{2} +  \samplesize \ell_{2} (\rho_{\rm min}/12) \big) } \bigg) \nonumber\\[2mm]
& \stackrel{\eqref{equ_events_3_m3}}{\leq} \exp \bigg(-\frac{  \samplesize \big(m_{3} +  \ell_{2} (\rho_{\rm min}/12) \big)}{16 \beta } \bigg) \nonumber\\[2mm]
& \stackrel{\eqref{equ_bound_m_3_111}}{\leq} \exp \bigg(-\frac{  \samplesize \big( (2/3) \ell_{1} \rho_{\rm min} +  \ell_{2} (\rho_{\rm min}/12) \big)}{16 \beta } \bigg) \nonumber\\[2mm]
& \leq \exp \bigg(-\frac{  \samplesize \rho_{\rm min}  \big( \ell_{1}+  \ell_{2}\big)}{192 \beta } \bigg).    
\label{equ_upper_bound_A_c_d_12}
\end{align}

In order to control the probability $\prob \{ \mathcal{E}_2(\delta) | \mathcal{A}, \vx_{\setT} \}$ appearing in \eqref{equ_bound_prob_E_2_total_prob}, 
we will again use Lemma \ref{lem_sub_exp_lin}. To this end, note that 
\begin{align}
\expect \bigg\{ \frac{1}{\samplesize} \sum_{j=1}^{\samplesize} v^{2}_{j} d^{2}_{j}\!+\!\frac{2}{\samplesize} \sum_{j=1}^{\samplesize} v_{j} d_{j} \varepsilon_{j} \bigg| \mathcal{A}, \vx_{\setT} \bigg\} \!
\stackrel{(a)}{=}\! \frac{1}{\samplesize} \sum_{j=1}^{\samplesize} d_{j}^{2} \stackrel{\eqref{equ_events_3_m3}}{=}  m_{3}, \label{equ_expect_e_2_m_3}
\end{align} 
with $(a)$ due to the random variables $\{v_{j}\}_{j=1,\ldots,\samplesize}$ being i.i.d standard Gaussian $\mathcal{N}(0,1)\}$, 
conditioned on $\vx_{\setT}$ and $\mathcal{A}$ (see \eqref{equ_def_event_A}). 
Then, 
\begin{align}
\label{equ_bound_cond_prob_E_2_A}
\prob \{ \mathcal{E}_2(\delta) | \mathcal{A}, \vx_{\setT} \} &\stackrel{\eqref{equ_event_2_23456},\eqref{equ_expect_e_2_m_3}}{\leq} 
\prob \bigg\{ \frac{1}{\samplesize} \sum_{j=1}^{\samplesize} v^{2}_{j} d^{2}_{j}\!+\!\frac{2}{\samplesize} \sum_{j=1}^{\samplesize} v_{j} d_{j} \varepsilon_{j}
- m_{3} \!\leq\! -(3/8)m_{3} -(\rho_{\rm min}/12) \ell_{2}  \mid \mathcal{A}, \vx_{\setT} \bigg\} \nonumber \\ 
  &\hspace*{-10mm} \leq  \prob \bigg\{ \big| \frac{1}{\samplesize} \sum_{j=1}^{\samplesize} v^{2}_{j} d^{2}_{j}\!+\!\frac{2}{\samplesize} 
  \sum_{j=1}^{\samplesize} v_{j} d_{j} \varepsilon_{j}- m_{3} \big| \!\geq\! (3/8)m_{3}\!+\!\rho_{\rm min} \ell_{2}/12  \mid \mathcal{A}, \vx_{\setT} \bigg\}\nonumber \\ 
  &\hspace*{-10mm} \leq  \prob \bigg\{ \big| \sum_{j=1}^{\samplesize} v^{2}_{j} d^{2}_{j}\!+\!2 \sum_{j=1}^{\samplesize} v_{j} d_{j} \varepsilon_{j}
  - \samplesize m_{3} \big| \!\geq\! \samplesize((3/8)m_{3}\!+\!\rho_{\rm min} \ell_{2}/12) \mid \mathcal{A}, \vx_{\setT} \bigg\}.
\end{align} 
Applying Lemma \ref{lem_sub_exp_lin} to \eqref{equ_bound_cond_prob_E_2_A}, using $\eta\!\defeq\!\samplesize(3m_{3}/8\!+\!\rho_{\rm min} \ell_{2}/12) $, 
$a_{j}\!\defeq\!d^2_{j}\!\stackrel{\eqref{equ_upper_bound_sing_values_beta}}{\leq} \!\beta$, $b_{j}\!\defeq\!d_{j} \varepsilon_{j}$ yields
\begin{align}
\label{equ_bound_cond_prob_E_2_B}
\prob \{ \mathcal{E}_2(\delta) | \mathcal{A} ,  \vx_{\setT} \} & \stackrel{\eqref{lem_1_1_lin}}{\leq} \nonumber \\[3mm]
& \hspace*{-10mm} 2 \exp\bigg(\!-\!\frac{ \samplesize^2 (3m_{3}/8\!+\!\rho_{\rm min} \ell_{2}/12)^2/8}{\beta (\sum_{j=1}^{\samplesize} d_{j}^{2} 
+ (1/\beta)\sum_{j=1}^{\samplesize} d_{j}^{2} \varepsilon_{j}^{2}+\samplesize(3m_{3}/8\!+\!\rho_{\rm min} \ell_{2}/12))} \bigg) \nonumber \\[3mm]
& \hspace*{-10mm}  \stackrel{\eqref{equ_upper_bound_sing_values_beta}}{\leq }2 \exp\bigg(\!-\!\frac{\samplesize^2 (3m_{3}/8\!+\!\rho_{\rm min} 
\ell_{2}/12)^2/8}{\beta ( \samplesize m_{3}\!+\!(1/\beta)\sum_{j=1}^{\samplesize} d_{j}^{2} \varepsilon_{j}^{2}+  \samplesize(3m_{3}/8\!+\!\rho_{\rm min} \ell_{2}/12))} \bigg) \nonumber \\[3mm]
& \hspace*{-10mm}  \stackrel{\eqref{equ_def_event_A}}{\leq }2 \exp\bigg(\!-\!\frac{\samplesize^2 (3m_{3}/8\!+\!\rho_{\rm min} \ell_{2}/12)^2/8}
{\beta ( 3 \samplesize m_{3}  + \samplesize \rho_{\rm min} \ell_{2}/12+\samplesize(3m_{3}/8\!+\!\rho_{\rm min} \ell_{2}/12))} \bigg) \nonumber \\[3mm]
& \hspace*{-10mm}  \leq  2 \exp\bigg(\!-\!\frac{\samplesize^2 (3m_{3}/8\!+\!\rho_{\rm min} \ell_{2}/12)^2/8}{\beta 9 \samplesize(3m_{3}/8\!+\!\rho_{\rm min} \ell_{2}/12)} \bigg) \nonumber \\[3mm]
& \hspace*{-10mm}  \leq  2 \exp\bigg(\!-\!\frac{\samplesize (3m_{3}/8\!+\!\rho_{\rm min} \ell_{2}/12)}{72 \beta} \bigg) \nonumber \\[3mm]
& \hspace*{-10mm}  \stackrel{\eqref{equ_bound_m_3_111}}{\leq }
2 \exp\bigg(\!-\!\frac{\samplesize \rho_{\rm min}(\ell_{1}\!+\!\ell_{2})/12 }{72 \beta  } \bigg). 
\end{align} 

By combining \eqref{equ_bound_cond_prob_E_2_B} and \eqref{equ_upper_bound_A_c_d_12} with \eqref{equ_bound_prob_E_2_total_prob}, 
\begin{align}
\prob \{ \mathcal{E}_2(\delta) \}  = \expect \{ \prob \{ \mathcal{E}_2(\delta) \mid \vx_{\setT} \} \}  &\leq   4 \exp\bigg(\!-\!\frac{\samplesize \rho_{\rm min}(\ell_{1}\!+\!\ell_{2})}{864 \beta  } \bigg)   
\label{equ_upper_bound_prob_E_2_12}
\end{align}

Summing \eqref{equ_upper_bound_E_1_CaseI} and \eqref{equ_upper_bound_prob_E_2_12} yields (cf. \eqref{equ_union_bound_E_1_2})
\begin{align}
\label{equ_bound_err_event}
\prob \{ \errevent \}  \leq M(\ell_{1},t) \defeq 6 \exp\bigg(\!-\!\frac{\samplesize \rho_{\rm min}(\ell_{1}\!+\!\ell_{2})}{ 864 \beta } \bigg).
\end{align} 
Inserting the upper bound \eqref{equ_bound_err_event} into \eqref{equ_prob_bound_M_N}, 
\begin{align}
\nonumber
\log  \prob \{ \error \}  \leq 2 \log \sparsity + \max_{(\ell_{1},t)\in \mathcal{I}} \hspace*{-0mm} \big[ (\ell_{1} + \ell_{2}) \log \coeffdim \!+\! \log 6 - \frac{\samplesize \rho_{\rm min}(\ell_{1}\!+\!\ell_{2})}{ 864 \beta} \big]. 
\end{align} 
Thus, $\prob \{ \error \}  \leq \eta$ whenever $\samplesize \geq 864 \log( \coeffdim 6 \sparsity^2 /\eta)  (\beta/\rho_{\rm min})$.

\vspace*{-1mm}
\section*{Appendix} 
The main device underlying our analysis is the following large deviation property of a quadratic form involving Gaussian random variables.
\begin{lemma}
\label{lem_sub_exp_lin}
Consider two vectors $\mathbf{a} =(a_{1},\ldots,a_{\samplesize})^{T} \in \mathbb{R}^{\samplesize}$ and $\mathbf{b} =(b_{1},\ldots,b_{\samplesize})^{T} \in \mathbb{R}^{\samplesize}$. 
For $\samplesize$ i.i.d. random variables $z_{j} \sim \mathcal{N}(0,1)$, define 
\begin{equation}
\label{equ_def_y_lem_sub_exp_lin} 
y\!=\!\sum_{j=1}^{\samplesize} a_{j} z_{j}^{2}\!+\!b_{j} z_{j}.
\end{equation} 
 
Then, 
\begin{equation}
\label{lem_1_1_lin}
\prob \{ |y - {\rm E}\{y\}|\!\geq\!\eta \}\!\leq\!2 \exp\bigg(\hspace*{-2mm}- \frac{\eta^{2}/8}{\| \mathbf{a} \|_{2}^{2}+ \| \mathbf{b} \|_{2}^{2} + \| \mathbf{a} \|_{\infty} \eta} \bigg).
\end{equation}
\end{lemma} 

\begin{proof}
An elementary calculation (see, e.g., \cite[Lemma 7.6]{RauhutFoucartCS}) reveals 
\begin{equation} 
\label{equ_expect_sub_exp_lin}
\hspace*{-3mm}\expect \{ \exp (\lambda (a_i  z_i^2\!+\!b_{i} z_{i} )\} \!=\! \exp \bigg( \hspace*{-1mm}\frac{\lambda^{2} b_{i}^{2}/2}{1\!-\!2\lambda a_{i}} \bigg) \sqrt{\frac{1}{1 \!-\! 2\lambda a_i }}, 
\end{equation} 
which holds for any $\lambda \!\in\! [0,1/(4\|\mathbf{a}\|_{\infty})]$. Hence, for any $i \in \{1,\ldots,\samplesize\}$,
\begin{align}
\label{lem_2_2}
&\log \expect \{ \exp (\lambda (a_i z_i^2 + b_{i}z_{i} - a_{i}))\}  \nonumber \\ 
& \stackrel{\eqref{equ_expect_sub_exp_lin}}{=} \frac{\lambda^{2} b_{i}^{2}/2}{1\!-\!2\lambda a_{i}} -(1/2)\log(1\!-\!2\lambda a_i)\!-\!\lambda a_i \nonumber\\
& \stackrel{\lambda|a_{i}|\leq1/4}{\leq} \lambda^{2} b_{i}^{2}- (1/2)\log(1 - 2\lambda |a_i|) - \lambda |a_i|.
\vspace*{-2mm}
\end{align}
Since $- \log (1-u) \leq u + \frac{u^2}{2(1-u)}$, for $0 \leq u \leq 1$, the RHS of \eqref{lem_2_2} yields, for every $i \in \nat{N}$,
\vspace*{-2mm}
\begin{align}
\label{equ_bound_log_expect_lambda_square_lin}
\log \expect \{ \exp (\lambda (a_i z_i^2 + b_{i}z_{i} - a_{i}))\} & \leq\lambda^{2} b_{i}^{2} + \frac{\lambda^2 a_i^2}{1-2\lambda |a_i|}  \nonumber \\ 
& \hspace{-15mm}\leq 2 \lambda^2 (a_i^2+(1/2)b_{i}^{2}).
\vspace*{-2mm}
\end{align}
Summing \eqref{equ_bound_log_expect_lambda_square_lin} for $i=1,\ldots,\samplesize$ and inserting into \eqref{equ_def_y_lem_sub_exp_lin},
\begin{align}
\label{lem_1_3_lin}
\log \expect \{ \exp ( \lambda(y-\expect\{y\}))\} \leq 2 \lambda^2 (\| \mathbf{a} \|_{2}^{2}+(1/2) \| \mathbf{b} \|_{2}^{2}).
\end{align}
Now, consider the tail bound (see, e.g., \cite[Remark 7.4]{RauhutFoucartCS})
\begin{align}
\label{lem_1_4_lin}
\prob\{ y - \expect\{y\} \geq \eta \} &\leq {\exp (-\lambda \eta)} \expect \{ \exp ( \lambda(y-\expect\{y\}))\} \nonumber\\
& \hspace*{-10mm} \stackrel{\equref{lem_1_3_lin}}{\leq} \exp (- \lambda \eta + 2 \lambda^2 (\| \mathbf{a} \|_{2}^{2}+(1/2) \| \mathbf{b} \|_{2}^{2}) ).
\end{align}
Minimizing the RHS of \eqref{lem_1_4_lin} over $\lambda\!\in\![0,1/(4\|\mathbf{a}\|_{\infty})]$, 
\begin{align}
\label{lem_1_5_lin}
\prob\{ y \!-\! \expect\{y\} \!\geq\! \eta \}
 &\!\leq\!
\exp \bigg( - \frac{\eta^{2}/8}{(\| \mathbf{a} \|_{2}^{2}+(1/2) \| \mathbf{b} \|_{2}^{2}) \vee (\eta \| \mathbf{a} \|_{\infty})}  \bigg)  \nonumber \\ 
& \hspace*{-20mm} \stackrel{(a)}{\leq} \exp \bigg(- \frac{\eta^2/8} {(\| \mathbf{a} \|^{2}_{2}+(1/2)\| \mathbf{b} \|_{2}^{2}) +  \| \mathbf{a} \|_{\infty} \eta}\bigg),
\end{align}
where $(a)$ is due to $ x \vee y \!\leq\!x\!+\!y$ for $x,y \!\in\! \mathbb{R}_{+}$. 
Similar to \eqref{lem_1_5_lin}, one can also verify 
\begin{equation}
\label{equ_bound_negative_sie_y_e_y_lin}
\prob\{ y - \expect\{y\} \leq -\eta \}\\
\leq \exp \bigg(- \frac{\eta^2/8} {(\| \mathbf{a} \|_{2}^{2}+(1/2)\| \mathbf{b} \|_{2}^{2})+ \| \mathbf{a} \|_{\infty}\eta} \bigg).
\end{equation}
Adding \eqref{lem_1_5_lin} and \eqref{equ_bound_negative_sie_y_e_y_lin} yields \eqref{lem_1_1_lin} by union bound. 
\end{proof}

\renewcommand{\baselinestretch}{0.9}\normalsize\footnotesize


\end{document}